\documentclass{article}

\usepackage{hyperref}
\usepackage{xcolor}
\definecolor{mygreen}{RGB}{53,134,80}

\usepackage[accepted]{icml2026}

\usepackage[utf8]{inputenc}
\usepackage[T1]{fontenc}

\usepackage{microtype}
\usepackage{graphicx}
\usepackage{booktabs}

\usepackage{amsmath}
\usepackage{amssymb}
\usepackage{mathtools}
\usepackage{amsthm}
\usepackage{amsfonts}

\usepackage{array}
\usepackage{multirow}
\usepackage{wrapfig}
\usepackage{subcaption}
\usepackage{enumitem}
\usepackage{soul}
\usepackage{float}
\usepackage{rotating}
\usepackage{makecell}
\usepackage{xspace}
\usepackage{comment}
\usepackage[normalem]{ulem}

\usepackage{algorithm}
\usepackage{algorithmic}

\usepackage{newunicodechar}
\newunicodechar{，}{,}

\usepackage{booktabs}
\usepackage{adjustbox}
\usepackage[table]{xcolor}
\usepackage{array}
\usepackage{graphicx}
\usepackage{makecell}
\usepackage{pgfmath}
\usepackage{pgf}

\newcolumntype{C}{>{\centering\arraybackslash}m{0.48cm}}

\newcommand{\BiasCell}[2]{%
  \begingroup
  \pgfmathsetmacro{\biasclip}{min(1,#2)}%
  \pgfmathtruncatemacro{\shade}{6 + 72*\biasclip}%
  \edef\temp{\noexpand\cellcolor{red!\shade}}%
  \temp\scriptsize #1%
  \endgroup
}

\newcommand{\TimeCell}[1]{%
  \begingroup
  \pgfmathsetmacro{\tnorm}{max(0,min(1,((#1)-0.05)/(22.78-0.05)))}%
  \pgfmathtruncatemacro{\shade}{8 + 72*\tnorm}%
  \edef\temp{\noexpand\cellcolor{blue!\shade}}%
  \temp\scriptsize #1%
  \endgroup
}

\usepackage{amsmath,amsfonts,bm}

\def\eqref#1{equation~\ref{#1}}

\def\1{\bm{1}}

\def\rvx{{\mathbf{x}}}
\def\rvy{{\mathbf{y}}}

\DeclareMathAlphabet{\mathsfit}{\encodingdefault}{\sfdefault}{m}{sl}
\SetMathAlphabet{\mathsfit}{bold}{\encodingdefault}{\sfdefault}{bx}{n}

\def\D{{\mathcal{D}}}

\theoremstyle{plain}
\newtheorem{theorem}{Theorem}[section]

\newtheorem{proposition}[theorem]{Proposition}
\newtheorem{corollary}[theorem]{Corollary}
\theoremstyle{remark}

\newcommand{\h}{\mathcal{H}}
\newcommand{\mi}{\mathbb{I}}
\newcommand{\smi}{\mathbb{SI}}

\newcommand{\p}{\mathrm{p}}

\newcommand{\yc}[1]{{\color{blue} #1}}

\newcommand{\method}{{\it InfoAtlas}\xspace}
\definecolor{asparagus}{rgb}{0.53, 0.66, 0.42}
\definecolor{ballblue}{rgb}{0.13, 0.67, 0.8}
\icmltitlerunning{InfoAtlas: A Foundation Model for Zero-Shot Statistical Dependence Estimate}

\begin{document}

\twocolumn[
\icmltitle{InfoAtlas: A Foundation Model for Zero-Shot 
Statistical Dependence Estimate}


\icmlsetsymbol{equal}{*}
\icmlsetsymbol{corr}{+}

\begin{icmlauthorlist}
\icmlauthor{Zhengyang Hu}{hku,equal}
\icmlauthor{Yanzhi Chen}{cam,ms,equal}
\icmlauthor{Hanxiang Ren}{zju}
\icmlauthor{Qunsong Zeng}{hku}
\icmlauthor{Youyi Zheng}{zju}
\icmlauthor{Adrian Weller}{cam,al}
\icmlauthor{Kaibin Huang}{hku}
\icmlauthor{Yanchao Yang}{hku,corr}
\end{icmlauthorlist}

\icmlaffiliation{hku}{The University of Hong Kong}
\icmlaffiliation{cam}{University of Cambridge}
\icmlaffiliation{ms}{Microsoft}
\icmlaffiliation{al}{Alan Turing Institute}
\icmlaffiliation{zju}{Zhejiang University}

\icmlcorrespondingauthor{Zhengyang Hu}{u3010250@connect.hku.hk}
\icmlcorrespondingauthor{Yanzhi Chen}{yc514@cam.ac.uk}
\icmlcorrespondingauthor{Yanchao Yang}{yanchaoy@hku.hk}
\vskip 0.3in
\icmlkeywords{Mutual Information, Statistical Dependency, Foundation Model}
]

\printAffiliationsAndNotice{\icmlEqualContribution}

\begin{abstract}
Measuring statistical dependency between high-dimensional random variables is a fundamental task in data science and machine learning. Neural mutual information (MI) estimators offer a promising avenue, but they typically require costly iterative optimization for each new dataset, making them impractical for real-time applications. We present \method, a foundation model-like architecture that eliminates this bottleneck by directly inferring MI in a single forward pass. Pretrained on large-scale synthetic data with rich dependence patterns, \method learns to identify diverse dependence structures and predict MI directly from the dataset. Comprehensive experiments demonstrate that \method matches state-of-the-art neural estimators in accuracy while achieving 100× speedup,  can flexibly handle varying dimensions and sample sizes through a single unified model, and generalizes effectively to complex, real-world scenarios. By reformulating MI estimation as an inference task, \method establishes a foundation for real-time dependency analysis. {{Project page}: \href{https://datou30.github.io/InfoAtlas-page/}{InfoAtlas-project}}
\end{abstract}

\section{Introduction}
\label{sec:intro}

Understanding statistical dependencies between variables is fundamental to data science and machine learning. Quantifying how variables influence each other uncovers hidden structures and causal mechanisms that drive complex systems. Applications span a wide range of domains: in healthcare, identifying dependencies between lifestyle factors and disease risks enables personalized prevention strategies \citep{du2024lifestyle}; in autonomous driving, modeling dependencies between sensor signals and road conditions improves safety and decision making \citep{maanpaa2025dense}; in biology, assessing the dependence between protein sequences reveals insights for understanding their functional relationship~\cite{gowri2024approximating}; and in robotics, maximizing statistical dependence between observational states is shown useful for policy discovery~\cite{zhou2024maxmi}.

Mutual information (MI) \citep{shannon1948mathematical} has long served as a principled measure for dependency, uniquely capturing complex nonlinear relationships for multivariate variables in interpretable units of bits. Its generality has made it a core tool in data analysis, generative modeling and representation learning \citep{chen2016infogan,oord2018representation, chen2022scalable}. However, computing MI from empirical samples is notoriously difficult: closed-form solutions exist only for certain distributions, and neural estimators \citep{belghazi2018mutual, choi2020regularized, franzese2023minde,tschannen2019mutual,chen2020neural, tsai2020neural} require costly gradient-based optimization for every dataset, making them impractical for real-time or large-scale applications.

\begin{figure*}[t!]
     \centering
     \includegraphics[width=0.995\textwidth]{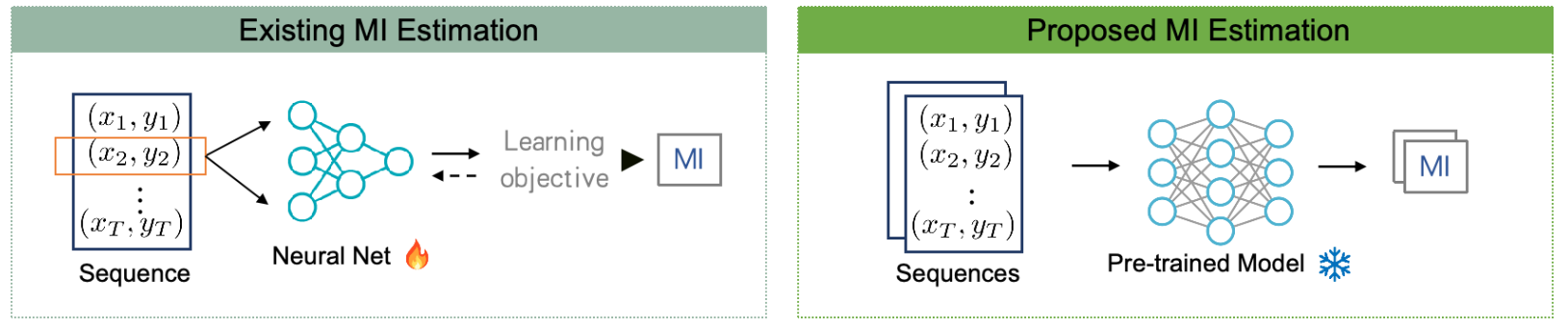}
     \caption{
     \textbf{Conceptual comparison: prior methods vs our method}.
     Existing neural MI estimators (left) requires iterative gradient-based optimization to train a neural network for each new dataset. In contrast, we uses a \emph{pre-trained} architecture to directly generate MI estimates in a single forward pass (right), eliminating per-dataset training and achieving speedup while maintaining comparable accuracy.}
     \label{fig:high-level-InfoNet}
\end{figure*}

In this work, we introduce \method, a foundation model-style architecture for fast and accurate estimation of statistical dependence between \emph{multivariate} random variables. \method predicts the strength of dependence in a single forward pass—an ability reminiscent of foundation models~\cite{hollmann2025accurate, comanici2025gemini}. This ability is acquired through large-scale pretraining on massive synthetic datasets that capture a wide range of patterns, plus a carefully design transfomer-based neural copula estimator

This ability is acquired through large-scale pretraining on massive synthetic datasets that capture a wide range of dependence structures and data patterns, enabling \method to directly infer statistical relationships without per-dataset optimization. Crucially, \method preserves full differentiability, facilitating seamless integration into larger AI pipelines. Extensive experiments demonstrate that \method generalizes effectively from synthetic benchmarks to complex real-world data, accurately capturing a broad spectrum of dependencies, being a versatile tool for rapid understanding of variable relationships. Our main contributions are:

\vspace{-0.24cm}

\begin{itemize}[leftmargin=*]
    \item We introduce \textbf{\method}, the first pretrained architecture for zero-shot estimation of mutual information between \emph{multivariate} variables. \method achieves accuracy on par with state-of-the-art neural methods without any gradient-based optimization, and flexibly handles variables of varying dimensionalities and sample sizes with a \emph{single}  model.
    \item We propose an attentive dual-path hypernetwork-based architecture, which is pretrained on large-scale synthetic datasets covering diverse dependency structures. This design enables \method to predict dependency strength in a single inference step, and generalizes effectively to unseen real-world scenarios without task-specific finetuning.
    \item We comprehensively evaluate \method on both synthetic benchmarks and real-world tasks, including independence testing, CLIP embedding analysis \citep{radford2021learning}, motion trajectory modeling and robotics manipulation. Results demonstrate its robust performance and accurate perception of a wide spectrum of dependencies.
\end{itemize}

\section{Problem Statement}
\label{sec:formulation}
In this work, we consider the problem of quantifying statistical dependence between two \emph{multivariate} random variables $\mathbf{x} \in \mathbb{R}^{d_x}$ and $\mathbf{y} \in \mathbb{R}^{d_y}$, with $d_x \geq 1$ and $d_y \geq 1$.

\paragraph{Measuring dependence via mutual information.}
Mutual information (MI) offers a principle measure for quantifying statistical dependence between multivariate variables. Unlike linear correlation coefficients that capture only linear relationships, MI effectively captures both linear and nonlinear correlations.
Formally, MI is defined as the Kullback-Leibler (KL) divergence between the joint distribution $p_{\mathbf{x},\mathbf{y}}$ and the product of marginals $p_{\mathbf{x}} \otimes p_{\mathbf{y}}$ \citep{kullback1997information}:
\begin{equation}
\begin{aligned}
    \mi(\mathbf{x},\mathbf{y}) &= \text{KL}(p_{\mathbf{x},\mathbf{y}}\|p_{\mathbf{x}} \otimes p_{\mathbf{y}}) \\ &= \int_{\mathcal{Y}}\int_{\mathcal{X}} p_{\mathbf{x},\mathbf{y}}(\mathbf{x}, \mathbf{y}) \log \left( \frac{p_{\mathbf{x},\mathbf{y}}(\mathbf{x}, \mathbf{y})}{p_{\mathbf{x}}(\mathbf{x}) p_{\mathbf{y}}(\mathbf{y})} \right) d\mathbf{x} d\mathbf{y}.
\end{aligned}
\end{equation}
Strong correlation manifests as significant divergence between $p(\mathbf{x}, \mathbf{y})$ and $p(\mathbf{x})p(\mathbf{y})$, yielding large MI, while uncorrelated variables satisfy $p(\mathbf{x}, \mathbf{y}) \approx p(\mathbf{x})p(\mathbf{y})$, resulting in MI near zero.

While MI offers a principled dependence measure,
it rarely admits closed-form solutions except for certain known distributions~\cite{czyz2023beyond, czyz2023properties}.
Thus, practical applications require estimation from finite samples $\mathcal{D} = \{\mathbf{x}^{i}, \mathbf{y}^{i}\}_{i=1}^n$ drawn from $p_{\mathbf{x},\mathbf{y}}$.
Recent advances have produced powerful neural estimators~\citep{belghazi2018mutual, duong2023diffeomorphic, franzese2023minde, tsai2020neural, poole2019variational, song2019understanding, letizia2024mutual, tsur2023neural}, with the most prominent one leveraging the Donsker-Varadhan (DV) representation~\citep{donsker1983asymptotic}:
\begin{align}
\mi(\mathbf{x},\mathbf{y}) \coloneqq \sup_{\theta} \mathbb{E}_{p_{\mathbf{x},\mathbf{y}}}[\theta] - \log(\mathbb{E}_{p_{\mathbf{x}} \otimes p_{\mathbf{y}}}[e^\theta]),
\label{eq:mi-dv}
\end{align}
where $\theta: \mathcal{X} \times \mathcal{Y} \rightarrow \mathbb{R}$ is a critic function.
Mutual Information Neural Estimation (MINE)~\citep{belghazi2018mutual} parametrizes $\theta$ as a neural network and approximates the supremum through gradient-based optimization. Besides MINE, there also exist a wide range of neural estimators based on different bounds and learning objectives; see \S\ref{sec:related}.

\paragraph{Challenge of real-time MI estimation.}
Despite differences in theory and algorithm, all existing neural estimators share a critical computational bottleneck in practice: they require training a network $\theta$ from scratch for each incoming dataset $\mathcal{D} = \{\mathbf{x}^{i}, \mathbf{y}^{i}\}_{i=1}^n$ via gradient descent:
\begin{equation}\vspace{-0.1em}
    \theta^{t+1} \leftarrow \theta^{t} - \eta \nabla_{\theta^{t}} \mathcal{L}(\theta^{t}), \quad t=1,...,T
\end{equation}
where $\mathcal{L}(\theta)$ is an estimator-specific objective (e.g., the negative DV bound for MINE). Achieving accurate MI estimates typically requires thousands of gradient steps, resulting in a $\mathcal{O}(T)$ computational complexity. This prohibitive cost limits real-time applications such as high-frequency financial correlation monitoring or large-scale genomic screening. The recent InfoNet \citep{hu2024infonet} aimed to address this inefficiency by pretraining a network to directly output optimal critic value via lookup tables, eliminating inference-time optimization.
However, InfoNet is fundamentally limited to univariate inputs, and extending  it to $d$-dimensional variables would require storing $\mathcal{O}(e^d)$ values in its lookup table, which quickly becomes intractable even for $d=8$, and it can not process data with varying data dimensionality. These limitations motivate our fundamentally different approach for real-time measurement of statistical dependence, where a unified model is developed to directly process multivariate data with varying dimensionalities and sample sizes.

\section{Method}
\label{sec:method}

We present \method, 
a pretrained architecture to address the above challenge of real-time correlation estimation 
between multivariate random variables $\mathbf{x} \in \mathbb{R}^{d_x}$ and $\mathbf{y} \in \mathbb{R}^{d_y}$. 
Unlike existing neural estimators that require iterative optimization, \method directly outputs mutual information (MI) in a single forward pass. This capability is enabled by two key innovations:
(i) a dual-path attentive hypernetwork-based architecture, which 
directly generates distribution-specific critic parameters from observed samples with varying sizes and dimensionality;
(ii) a comprehensive pre-training strategy using diverse synthetic  distributions, which ensures generalization across different application domains.

\subsection{Direct Optimal Critic Generation}

Our key innovation is to reformulate MI estimation 
from a test-time optimization problem 
into a direct inference task with the aid of a hypernetwork. Specifically, given a dataset $\mathcal{D} = \{(\mathbf{x}^i, \mathbf{y}^i)\}_{i=1}^n$ drawn from an unknown joint distribution, \method employs an attention-based hypernetwork $\mathcal{H}: \mathcal{D} \mapsto \Theta$ that directly outputs the complete parameter set $\theta^*$ of the optimal critic network in the Donsker-Varadhan representation (Eq.~\ref{eq:mi-dv}) via a single feedforward pass\footnote{
While primarily focusing on the DV representation, which we find useful for achieving good performance when paired with massive pre-training, our method is  fully compatible with other variational  estimators~\citep{song2019understanding, letizia2024mutual}.}:
\begin{equation}
    \theta^* = \mathcal{H}(\mathcal{D}) = \mathcal{H}( \{(\mathbf{x}^i, \mathbf{y}^i) \}_{i=1}^n)
\end{equation}
An empirical MI estimation is then obtained via
\begin{equation}
\hat{\mi}_{\theta}(\mathbf{x}, \mathbf{y}) 
= \frac{1}{n} \sum_{i=1}^{n} 
\theta(\mathbf{x}^i, \mathbf{y}^i) - 
\log \Big( \frac{1}{n} \sum_{j=1}^{n} e^{
\theta(\mathbf{x}^j, \mathbf{y}^{\pi(j)})}\Big),
\label{eq:empirical-mi}
\end{equation}
where $\{(\mathbf{x}^j, \mathbf{y}^{\pi(j)})\}_{j=1}^n$ 
denotes the marginal pairs with $\pi$ a random permutation of indices $\{1, ..., n\}$. This eliminates the iterative gradient updates required by neural MI estimators while avoiding the exponential value storage of InfoNet's lookup table approach. This architectural shift fundamentally changes the computational complexity from $\mathcal{O}(T)$ gradient steps, where $T$ is the number of optimization iterations, to $\mathcal{O}(1)$ feedforward propagation.

\begin{figure*}[!t]
    \centering
    \includegraphics[width=0.985\textwidth]{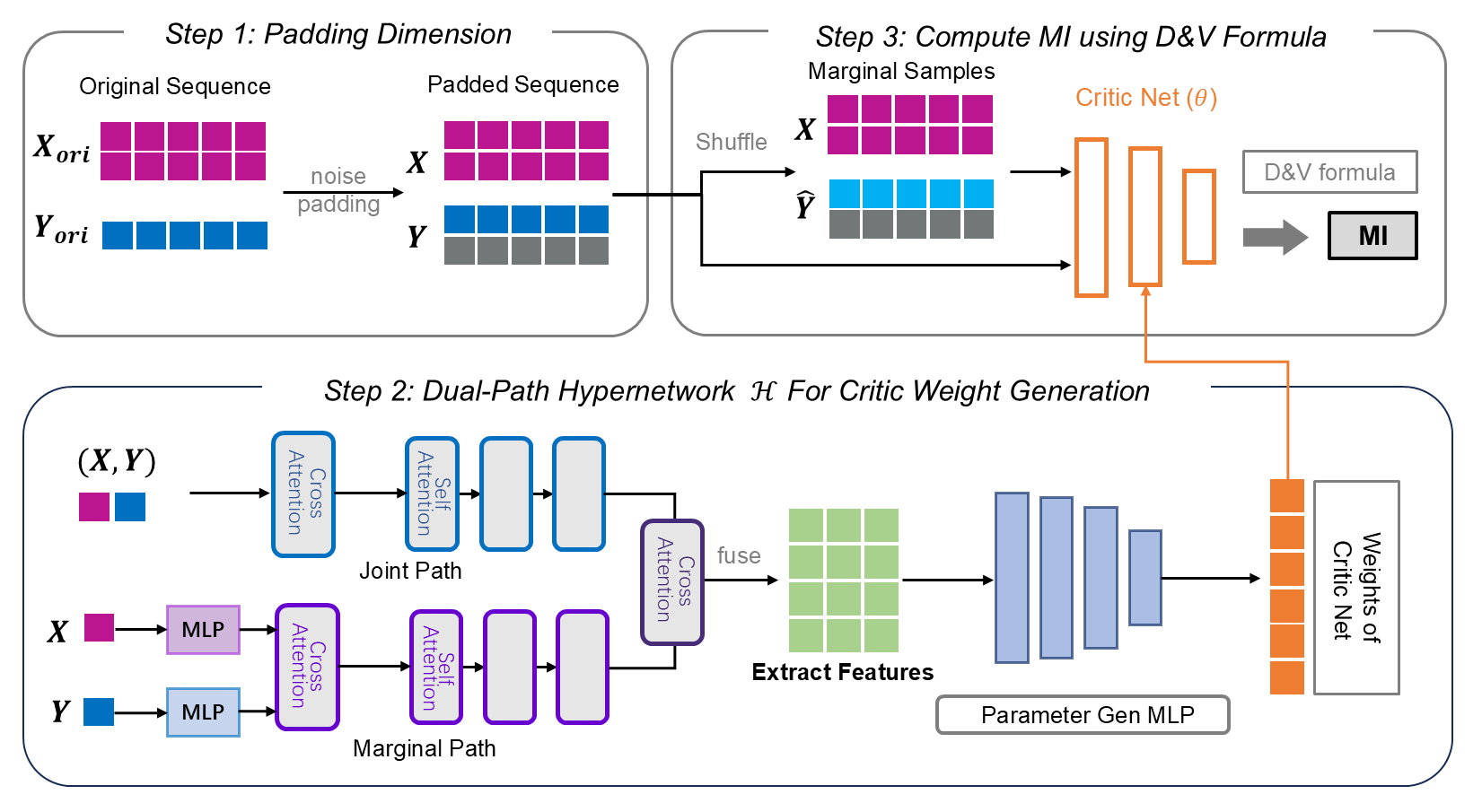}
    \caption{\textbf{The {\method} estimation pipeline}. Step 1: We pad input dimensions with noise to ensure all variables share the same  dimensionality, while allowing flexible sample sizes. Step 2: A dual-path hypernetwork $\h$—with joint and marginal branches—extracts features in alignment with the D-V formulation (Eq.~\ref{eq:mi-dv}). Cross-attention integrates these features, and a parameter-generation MLP is then used to produce the critic parameters. Step 3: The empirical D-V formula (Eq.~\ref{eq:empirical-mi}) is applied to joint and marginal samples, with marginals obtained by index permutation, to estimate MI. This pipeline enables single-pass estimation without gradient-based optimization.}
    \label{fig:architectures-mi} 
\end{figure*}

The hypernetwork $\mathcal{H}$ leverages attention~\cite{vaswani2017attention} and consists of the following key modules:

\paragraph{The joint distribution path} 
processes $n$ paired samples $\{(\mathbf{x}^i, \mathbf{y}^i)\}_{i=1}^n$ to extract correlation patterns inherent in $p(\mathbf{x}, \mathbf{y})$. 
Each sample pair is treated as a token in a sequence, enabling permutation-invariant processing through attention mechanisms. 
Specifically, a learnable query vector $\mathbf{q}_{\text{joint}} \in \mathbb{R}^{d_{\text{model}}}$ initiates cross-attention computation, where the concatenated samples $[\mathbf{x}^i; \mathbf{y}^i]$ serve simultaneously as keys and values. 
This mechanism computes attention weights $\alpha_i = \text{softmax}(\mathbf{q}_{\text{joint}}^T \mathbf{W}_K [\mathbf{x}^i; \mathbf{y}^i] / \sqrt{d_{\text{model}}})$, producing an aggregated representation that emphasizes sample pairs exhibiting strong correlations. 
The aggregated features subsequently pass through 16 self-attention layers ultimately producing a comprehensive encoding $\mathbf{h}_{\text{joint}} \in \mathbb{R}^{d_{\text{hidden}}}$ that characterizes the joint distribution's correlation structure.

\paragraph{The marginal distribution path} 
processes samples from the product of marginals $p(\mathbf{x})p(\mathbf{y})$ by breaking the pairing relationship. 
Specifically, the samples $\{\mathbf{x}^i\}_{i=1}^n$ and $\{\mathbf{y}^j\}_{j=1}^n$ are passed through separate projection networks $f_{\mathbf{x}}: \mathbb{R}^{d_x} \to \mathbb{R}^{d_{\text{proj}}}$ and $f_{\mathbf{y}}: \mathbb{R}^{d_y} \to \mathbb{R}^{d_{\text{proj}}}$, implemented as Multi-Layer Perceptrons (MLPs) to obtain higher-dimensional representations that facilitate correlation detection. 
The architecture employs bidirectional cross-attention with two sets of learnable query vectors: $\mathbf{q}_{\mathbf{x} \to \mathbf{y}}$ attends from projected $\mathbf{x}$ representations to projected $\mathbf{y}$ representations (as keys and values), while $\mathbf{q}_{\mathbf{y} \to \mathbf{x}}$ performs the reverse attention. 
The outputs from both directions are summed element-wise, capturing symmetric independence patterns that should appear when variables lack correlation. 
This combined representation is then processed by 8 self-attention layers, resulting in encoding $\mathbf{h}_{\text{marginal}} \in \mathbb{R}^{d_{\text{hidden}}}$ that provides a baseline representation against which the correlation strength can be measured.

\paragraph{The integration and generation module} 
fuses information from both distributional paths through a cross-attention mechanism that allows the joint distribution features to be modulated by marginal distribution patterns. 
In particular, we compute cross-attention where $\mathbf{h}_{\text{marginal}}$ serves as the query and $\mathbf{h}_{\text{joint}}$ provides both keys and values, producing a fused representation $\mathbf{h}_{\text{fused}} = \text{CrossAttention}(\mathbf{h}_{\text{marginal}}, \mathbf{h}_{\text{joint}}, \mathbf{h}_{\text{joint}})$. 
This asymmetric fusion ensures that correlation patterns identified in the joint path are evaluated against the independence baseline from the marginal path. 
The fused features are then processed by a parameter generation MLP serving as a nonlinear mapping from distributional features to critic network parameters. 
The MLP outputs a flattened vector $\theta \in \mathbb{R}^{|\Theta|}$ containing all parameters for a critic network, where $|\Theta| = \sum_{l=1}^{L} (d_l \times d_{l-1} + d_l)$ accounts for both weight matrices and bias terms across all layers.

\paragraph{Noise padding module} further 
addresses the challenge of varying input dimensions through a unified data preprocessing strategy that maintains MI while enabling consistent model architecture. 
For inputs with dimensions $d < D$, we pad variables with independent Gaussian noise $\mathcal{N}(0, \mathbf{I})$ to reach $D$ dimensions. 
This padding preserves mutual information exactly since $\mi(\mathbf{x}, \mathbf{y}) = \mi([\mathbf{x}; \mathbf{n}_x], [\mathbf{y}; \mathbf{n}_y])$ when the noise vectors $\mathbf{n}_x$ and $\mathbf{n}_y$ are independent of each other and independent of both $\mathbf{x}$ and $\mathbf{y}$, as justified in Proposition~\ref{prop:mi-invariance-extended}.

\subsection{Large-Scale Synthetic Pre-training}
We pretrain \method using a comprehensive spectrum of pretraining data (the `atlas'). For this purpose, we construct a meta-distribution $p(\mathcal{D})$ over datasets by systematically generating synthetic data $\mathcal{D}$ that span diverse statistical properties, drawing from the principle that a model pretrained on sufficiently diverse synthetic data can generalize effectively to real-world unseen data.

\paragraph{Diversity-driven synthetic distribution generation} Our diversity-aware data generation procedure consists of two complementary steps targeting the diversity of \emph{dependence structure} and \emph{marginal patterns} respectively:

\textit{Dependence diversity via random copula mixture}.
We ensure diversity in correlation structure by sampling from a diverse mixtures of copulas with varying dependence properties. 
Specifically, let $c_i$ be a copula chosen from a pre-defined pool $\mathcal{C}$. 
We generate samples $\mathbf{x}, \mathbf{y}$ according to:
\begin{equation}
    \mathbf{x}, \mathbf{y} \sim \sum_{i=1}^{K}\pi_i c_i, \quad c_i \in \mathcal{C}
\end{equation}
where the parameters of each copula $c_i$ and the mixture coefficients $\pi_i$ are randomly initialized.
In this work, we employ both Gaussian copulas with rich correlation matrix and Student's $t$-copulas with varying tail dependencies (see Appendix \ref{Appendix:Syn data gen}). According to recent vector copula theory~\citep{chen2025neural}, such copula mixture is a consistent estimator for the dependence structure between $\mathbf{x}$ and $\mathbf{y}$ given sufficiently large $K$. In this work, we use $K=60$ mixtures, substantially beyond $K=32$ used in~\citep{chen2025neural} for accurate approximation of dependence structure.

\textit{Marginal diversity via random flow transformation.}
To complement the correlation diversity, we enhance marginal pattern diversity through flow-based models \citep{papamakarios2021normalizing, dinh2016density}. 
Specifically, we apply two flow models $f_X: \mathbb{R}^{d_X} \to \mathbb{R}^{d_X}$, $f_Y: \mathbb{R}^{d_Y} \to \mathbb{R}^{d_Y}$ with randomly initialized parameters to transform data in each training batch:
\begin{equation}
    \mathbf{x} \leftarrow f_X(\mathbf{x}), \qquad \mathbf{y} \leftarrow f_Y(\mathbf{y}).
\end{equation}
These invertible transformations preserve mutual information while introducing complex marginal patterns, as $\mi(\mathbf{x}, \mathbf{y}) = \mi(f_{X}(\mathbf{x}), f_{Y}(\mathbf{y}))$ for any bijective function.
Additionally, we apply a differentiable copula transformation using the softrank function \citep{blondel2020fast}, which maps each marginal distribution to approximately uniform distribution $[0,1]$. This normalization enables the model to focus on learning correlation patterns rather than adapting to spurious features irrelevant to the true dependence structure.

\paragraph{Overall pre-training objective}
With comprehensive data generation, 
the parameters of the hypernetwork $\mathcal{H}$ are optimized 
through a meta-learning objective that maximizes the expected accuracy of MI estimation in the distribution of training datasets. 
Formally, we minimize:
\begin{equation}
\mathcal{L}(\mathcal{H}) = -\mathbb{E}_{\mathcal{D} \sim p(\mathcal{D})} \left[ \hat{\mi}_{\mathcal{H}(\mathcal{D})}(\mathbf{x}_{\mathcal{D}}, \mathbf{y}_{\mathcal{D}}) \right],
\label{eq:meta-objective}
\end{equation}
where $p(\mathcal{D})$ represents the meta-distribution over datasets induced by our synthetic generation process, 
$\mathcal{H}(\mathcal{D})$ outputs the critic parameters for dataset $\mathcal{D}$, and $\hat{\mi}_{\theta}(\mathbf{x}_{\mathcal{D}}, \mathbf{y}_{\mathcal{D}})$ is the empirical MI estimate using critic parameters $\theta$ as in Eq.~\ref{eq:empirical-mi}.

Proposition~\ref{prop:pop-correct} establishes that under mild conditions, the above learning objective yields a \emph{consistent} estimate to the ground truth MI for all $\D$ such that $p(\D) > 0$, thereby guaranteeing convergence to the optimal critic $\theta^*$.

We highlight two key differences between the above pre-training pipeline and conventional MI estimator training.
First, unlike standard approaches, our synthetic pre-training can generate \emph{unlimited} training datasets, yielding a theoretically infinite sample size per each dataset and a large batch size. Second, our pre-training amortizes learning across dataset through a centralized hypernetwork $\mathcal{H}$, enabling knowledge acquired from one dataset to effectively transfer to others. Together, we alleviate known failure modes of conventional neural MI estimation, such as high estimation variances and biases caused by insufficient  samples per dataset~\citep{mcallester2020formal, song2019understanding}.

\section{Scaling to High Dimensions via Slicing}

\begin{figure*}[!t]
    \centering
    \begin{subfigure}{.32\textwidth}
        \centering
        \includegraphics[width=\textwidth]{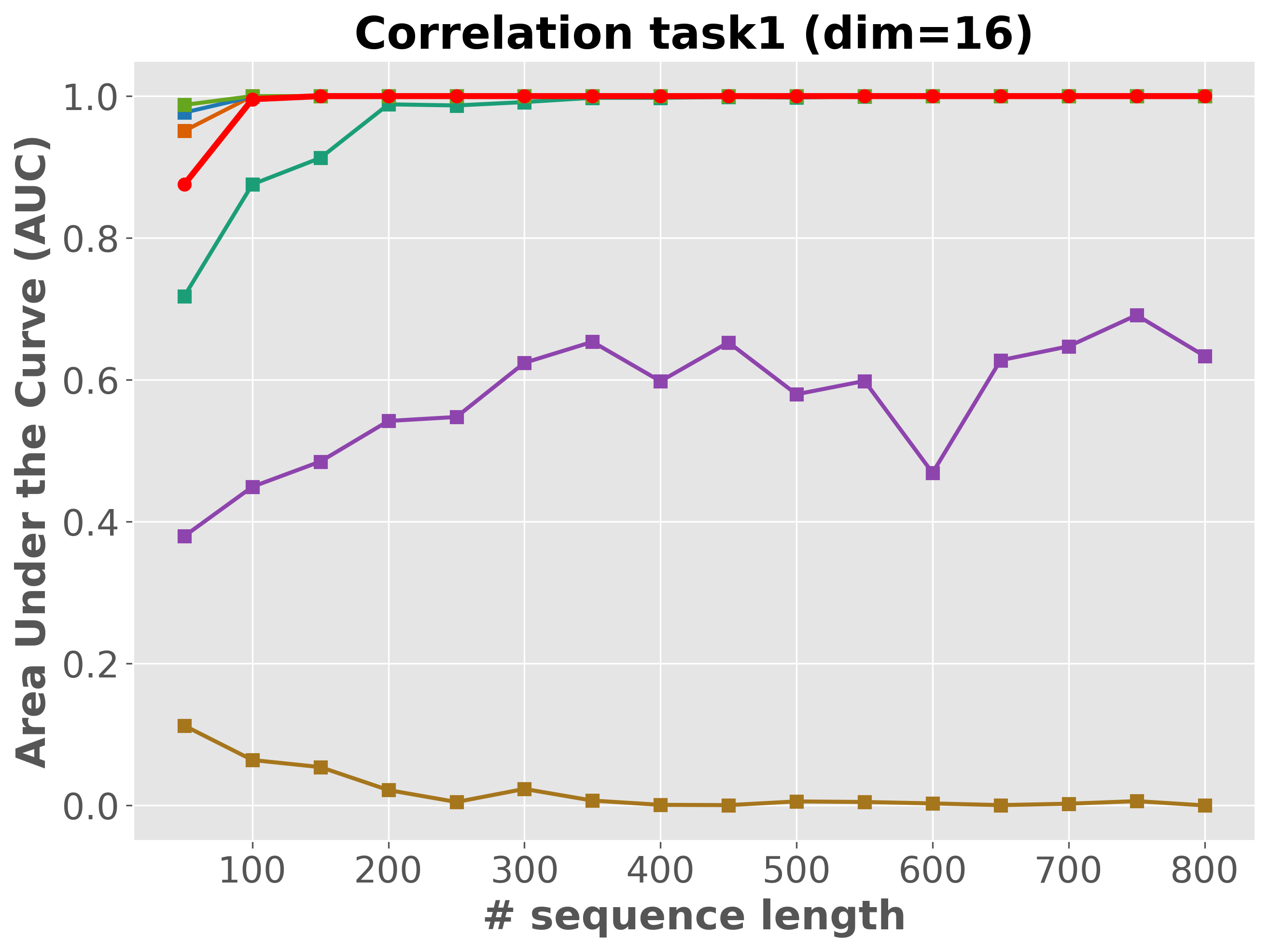}
        \vspace{-2mm}
    \end{subfigure}
    \begin{subfigure}{.32\textwidth}
        \centering
        \includegraphics[width=\textwidth]{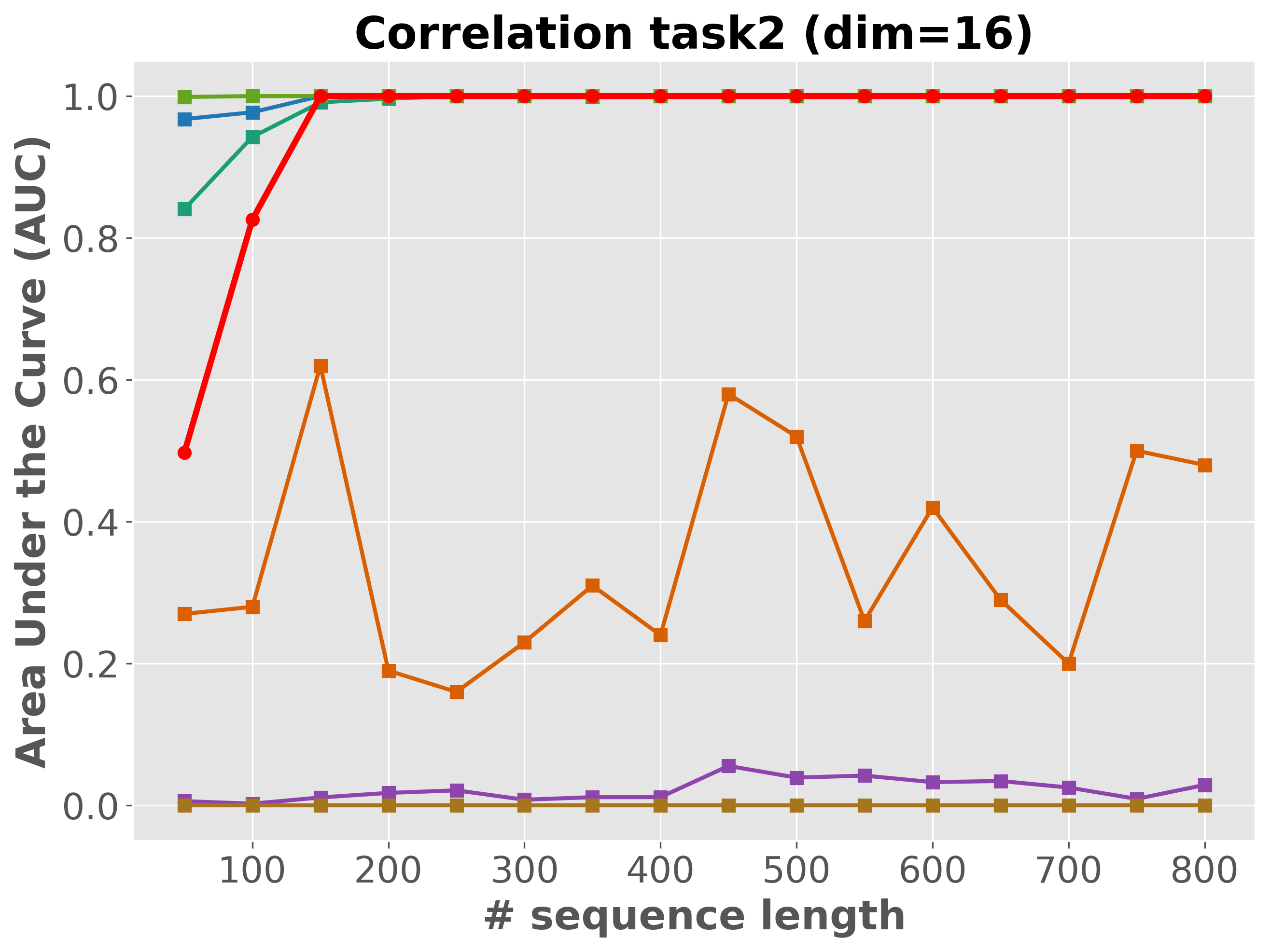}
        \vspace{-2mm}
    \end{subfigure}
    \begin{subfigure}{.32\textwidth}
        \centering
        \includegraphics[width=\textwidth]{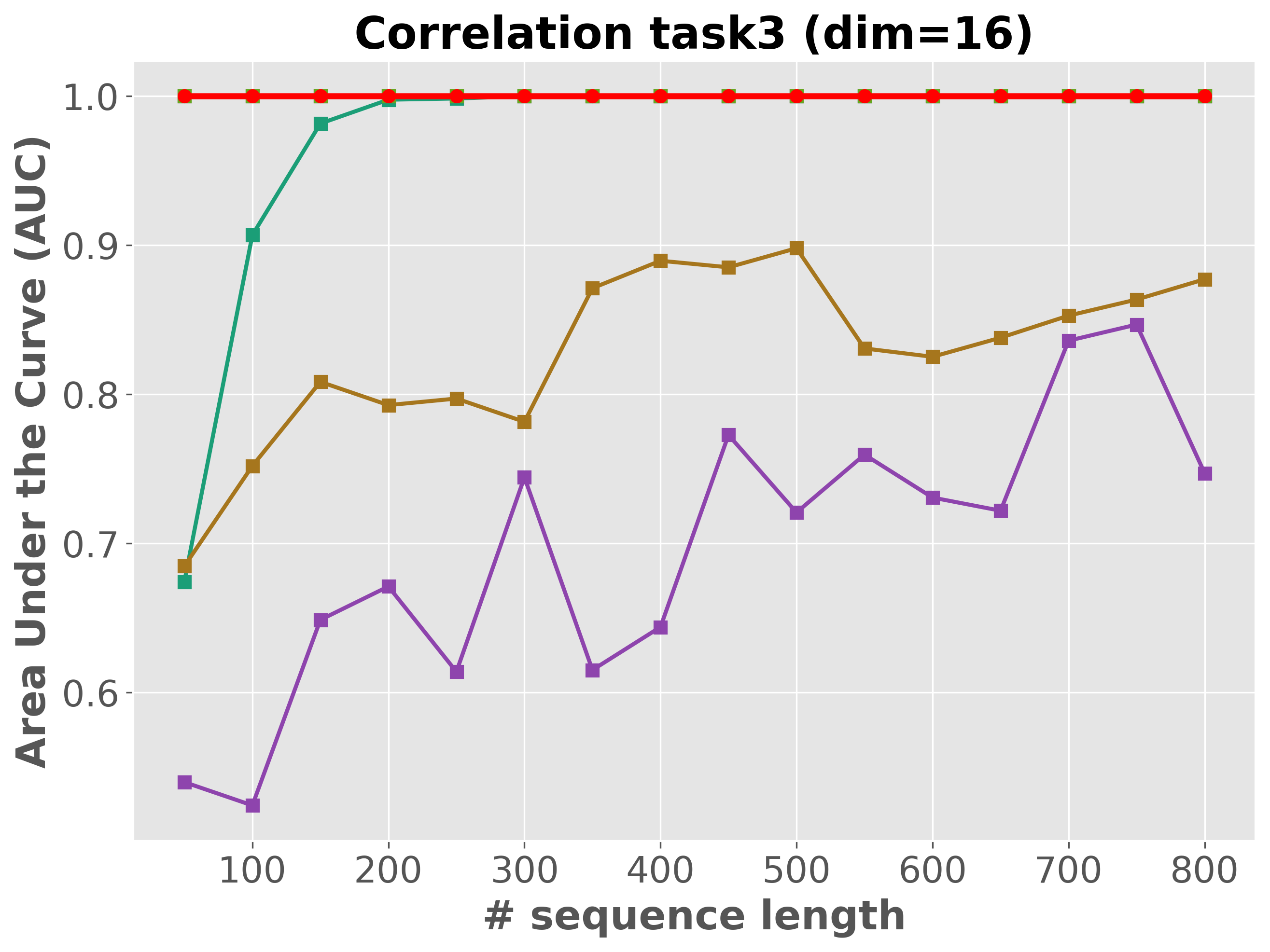}
        \vspace{-2mm}
    \end{subfigure}

    \begin{subfigure}{.32\textwidth}
        \centering
        \includegraphics[width=\textwidth]{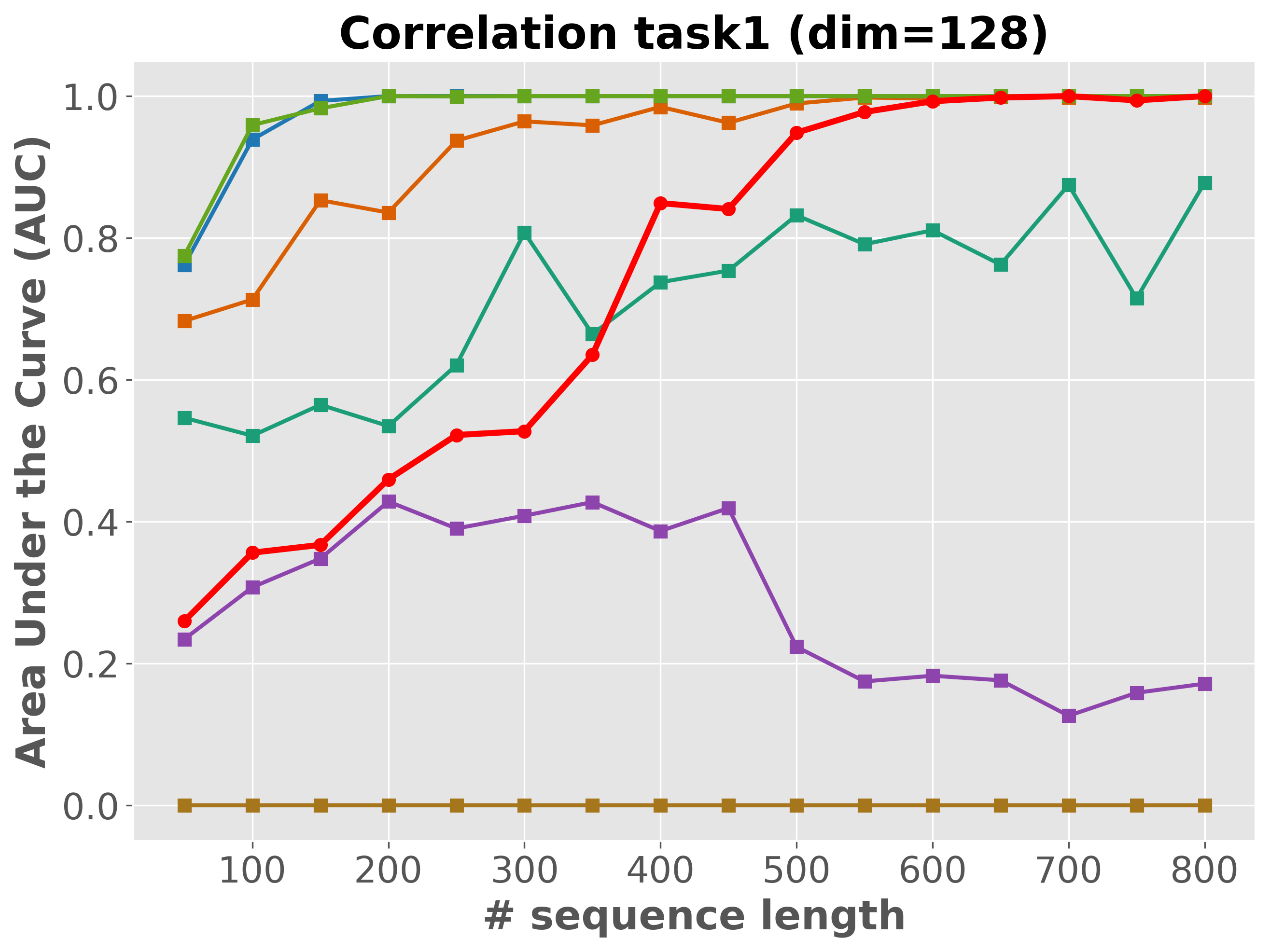}
        \vspace{-2mm}
    \end{subfigure}
    \begin{subfigure}{.32\textwidth}
        \centering
        \includegraphics[width=\textwidth]{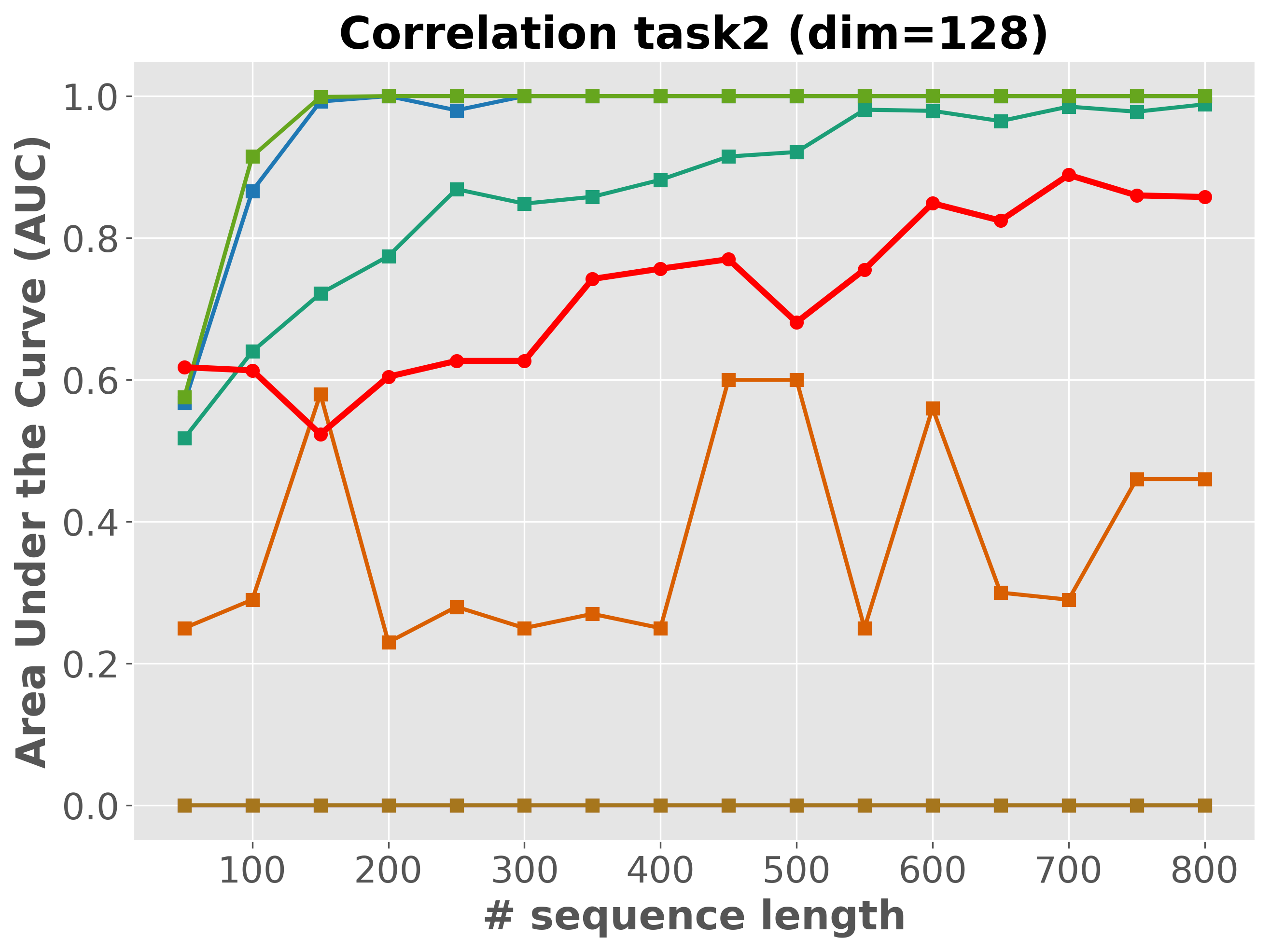}
        \vspace{-2mm}
    \end{subfigure}
    \begin{subfigure}{.32\textwidth}
        \centering
        \includegraphics[width=\textwidth]{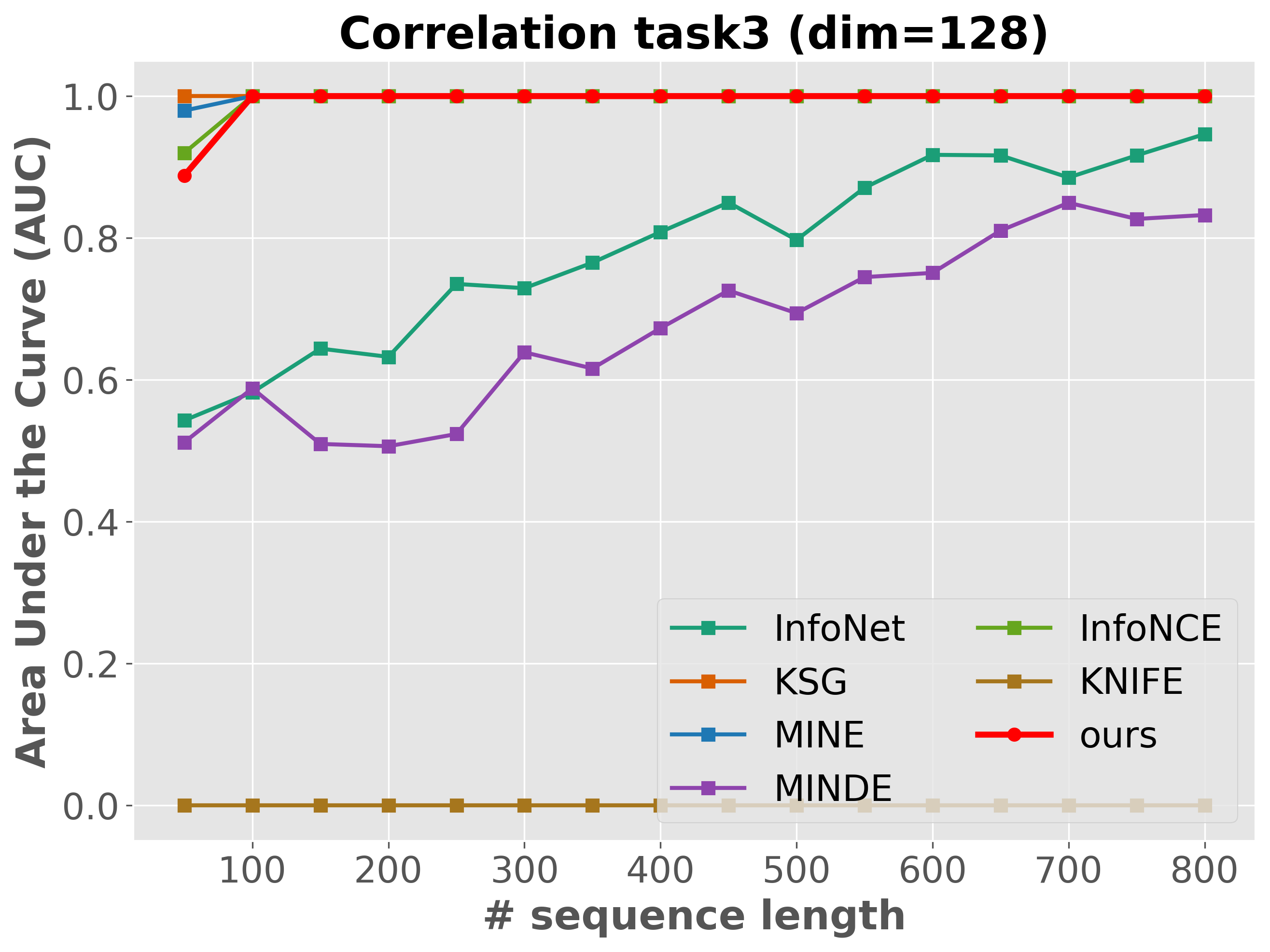}
        \vspace{-2mm}
    \end{subfigure}

    \caption{\textbf{Independence testing under three types of data correlations}. Each curve depicts the area under the curve (AUC) of the receiver operating characteristic (ROC) with respect to sequence length $n$. Seven MI estimators are compared: \method, InfoNet, KSG, MINE, MINDE, InfoNCE and KNIFE. \method uses 5-sliced MI with 32 slices, while InfoNet adopts 1-sliced MI with 128 slices. }
    \label{fig:independent-test}
\end{figure*}

\method natively supports multivariate inputs up to a predefined dimensionality $D$. To scale to higher-dimensional data with $d > D$, we leverage sliced mutual information (sliced MI), which estimates high-dimensional statistical dependence by aggregating information across multiple low-dimensional projections, referred to as ``slices''.

\paragraph{Slices of high-dimensional dependence}
Our intuition is as follows. In a geographical atlas, the information of a specific region may be characterized by different aspects, such as topology, climate or population. Similarly, in a genetic atlas, different slices may reveal gene activity across tissues, cell types, or spatial locations. Each individual view is necessarily partial, but they together complementarily provide a rich characterization of the underlying object.

We apply the same principle to high-dimensional dependence estimation. Rather than estimating MI directly in the original high-dimensional space, we examine many low-dimensional projections of the variables. Each projection captures one specific ``view'' of the original dependence structure, and aggregating over many such views yields an informative summary of the overall statistical dependence.

Formally, the ($k$-)sliced MI under slicing dimensionality $k$ is defined as~\cite{goldfeld2022k}:
\begin{equation}
\begin{aligned}
\smi_k(\mathbf{x};\mathbf{y})
&=
\mathbb{E}_{\mathbf{P},\mathbf{P'}}
\left[
\mi(\mathbf{P}^{\top}\mathbf{x};\mathbf{P}'^{\top}\mathbf{y})
\right], \\
&\approx \frac{1}{S}
\sum_{j=1}^{S}
\mi(\mathbf{P}_j^{\top}\mathbf{x};\mathbf{P}'^{\top}_j\mathbf{y})
\end{aligned}
\label{eq:k-smi}
\end{equation}
where $\mathbf{P}\in\mathrm{St}(d_x,k)$ and $\mathbf{P'}\in\mathrm{St}(d_y,k)$ are random orthonormal projection matrices sampled from the Stiefel manifolds. Note that single-sided slicing could also be used: $\smi'_k(\mathbf{x};\mathbf{y}) = 
\mathbb{E}_{\mathbf{P}}
\left[
\mi(\mathbf{P}^{\top}\mathbf{x};\mathbf{y})
\right]$.

Sliced MI has been shown to be an effective measure for quantifying high-dimensional  dependence, particularly under moderately large $k$ and $S$. It inherits several key properties of MI, including: (a) $\mi = 0 \Leftrightarrow \smi = 0$~\cite{goldfeld2021sliced}, and (b) maximizing $\mi$ is equivalent to maximizing $\smi$ under single-sided slicing~\cite{chen2023learning}. These properties make sliced MI a theoretically grounded and practically useful alternative to full MI in applications such as independence testing~\cite{tsur2023max, goldfeld2022k} and representation learning~\cite{chen2023learning, zhou2024maxmi}. In many applications, exact MI values is often not of interests; the relative strength of dependence as captured by sliced MI is already highly informative.

\paragraph{Batched inference across slices}
\method is particularly well suited for sliced MI estimation due to its ability to process multiple slices \emph{in parallel}. Existing neural MI estimators typically require a separate optimization procedure for each projection direction, leading to a time complexity of $O(ST)$, where $T$ denotes the optimization cost, e.g., the number of gradient steps. In contrast, the transformer-based architecture of \method allows multiple projected datasets to be packaged into a single batch and processed jointly through one feedforward pass, analogous to how large language models process multiple sequences simultaneously:
\begin{equation}
    \{\theta_1^*, ... \theta^*_S\} = \mathcal{H}(\{\mathcal{D}_j\}^S_{j=1}) 
\end{equation}
where $\theta^{*}_j$ denotes the optimal critic predicted by \method for the $j$-th slicing direction and $\mathcal{D}_j =\{\mathbf{P}_j\mathbf{x}^i, \mathbf{P}_j'\mathbf{y}^i \}^n_{i=1}$ denotes the projected dataset for that direction. This way, we reduce time complexity from $O(ST)$ to $O(1)$, facilitating highly efficient computation of sliced MI.

We note that while theoretically and empirically grounded, sliced MI is not a drop-in replacement for full MI: slicing may miss certain dependence structures, and finite-$S$ averaging can miss rare but informative projection directions. We discuss the failure cases of slicing, together with the effects of $k$, $S$ in Appendix~\ref{app:smi-failures} and ~\ref{app:sensitivity}. Nevertheless, as our slicing-based method strikes a favorable accuracy--efficiency trade-off across a broad range of settings, including high-dimensional real-world tasks, we consider it a practical and scalable approach for statistical dependence measurement.

\section{Experiments}
\label{sec:exps}

\subsection{Setups}


\textbf{Slicing}. \method is pretrained for data with dimensionality up to $D=20$. For high-dimensional inputs where $d > 20$, we employ $k$-sliced mutual information \citep{goldfeld2022k}, which projects the data onto $k$-dimensional random subspaces and averages the MI estimates across projections. 
Specifically, we compute $\hat{\smi}_k(\mathbf{x}, \mathbf{y}) = \frac{1}{S}\sum_{i=1}^{S} \hat{\text{I}}(\mathbf{P}_i\mathbf{x}, \mathbf{P}'_i\mathbf{y})$ where $\mathbf{P}_i, \mathbf{P}'_i \in \mathbb{R}^{k \times d}$ are random projection matrices. 
This approach preserves substantially more correlation structure than 1-dimensional slicing used in previous work~\citep{hu2024infonet}. For other neural methods such as MINDE, we do not employ slicing, as this will involve training $S$ networks for $S$ slicing directions, being computationally prohibitive.

\vspace{0.1cm}

\textbf{Baselines}. We consider seven MI estimation methods:
KSG \citep{kraskov2004estimating} and
KDE \citep{silverman2018density}, two classic non-parametric MI estimators that do not require costly iterative optimization;
KNIFE \citep{pichler2022differential}, which uses KDE with learnable parameters in MI estimate, and neural methods MINE \citep{belghazi2018mutual}, InfoNCE \citep{oord2018representation}, and MINDE \citep{franzese2023minde},
all of which require training a network from scratch for each new distribution. We also compare to the pretrained InfoNet~\citep{hu2024infonet} method whenever appropriate. InfoNet \citep{hu2024infonet} is restricted to one dimension,
thus requiring using slicing for data dimension beyond $D=1$.


\begin{table*}[!ht]
\centering
\caption{\textbf{Sanity check with BMI benchmark}. We compare \method with other MI estimators on 8 representative tasks
from \citep{czyz2023beyond} with known ground truth MI.
Each estimate represents the average over 10 random seeds
with $N = 5000$ samples per task.
Task notation indicates distribution type (Mn=Multivariate normal, St=Student-$t$, Asinh=Arc sinh, Uniform=correlated uniform, Hc=Half cube) and corresponding parameters, with the first two digits indicating dimensionality of $\rvx$ and $\rvy$ respectively.
Methods are color-coded: neural-based methods in {\color{asparagus}green} and non-neural methods in {\color{ballblue}blue}.
\textbf{Bold} indicates closest to ground truth,
while \underline{underlined} values show second-best estimates.
The rightmost column shows execution time in seconds to compare computational efficiency.
}
\vspace{1mm}
\label{tab:bmi}
\begin{tabular}{l *{8}{c} c}
\toprule
 & \multicolumn{8}{c}{Tasks} & \\
\cmidrule(lr){2-9}
\multirow{1}{*}{Method*} &
\makecell{Mn-dense\\5-5-0.5} &
\makecell{Spiral\\3-3-2-2.0} &
\makecell{Asinh@St\\5-5-2} &
\makecell{St\\3-3-3} &
\makecell{Uniform\\3-3-2-2.0} &
\makecell{Hc@Mn\\5-5-2} &
\makecell{Additive\\1-1-0.1} &
\makecell{Bimodal\\1-1-0.75} &
\multirow{1}{*}{Time (s)} \\
\midrule
\textit{GT}  & 0.59  & 1.02 & 0.45 & 0.18 & 1.02 & 1.02 & 1.71 & 0.41 & -- \\
\midrule
{\color{ballblue}KSG}  & 0.54 & 0.75 & 0.25  & 0.07 & 0.79 & 0.58 & 1.61 & \textbf{0.41} & \underline{0.13} \\
{\color{ballblue}KDE}  & 1.59 & 2.87 & 2.43  & 2.36 & 1.17 & 2.23 & 2.94 & 1.23 & 2.04 \\
{\color{asparagus}MINE}  & \textbf{0.60} & \textbf{1.00} & 0.53  & \underline{0.21} & \textbf{1.03} & \underline{1.06} & \textbf{1.63} & 0.39 & 25.9 \\
{\color{asparagus}MINE-5s}  & \textbf{0.60} & 0.90 &  0.33 & 0.15 & 0.93 & \underline{1.06} & 1.61 & 0.38 & 4.92 \\
{\color{asparagus}MINDE} & \textbf{0.58} & 0.92 & \textbf{0.43}  & 0.36 & 0.89 & \textbf{1.01} & 1.42 & 0.50 & 34.2 \\
{\color{asparagus}InfoNCE} & \underline{0.56} & \underline{0.98} & \underline{0.49}  & \textbf{0.18} & \underline{0.97} & \textbf{1.03} & \underline{1.62} & \underline{0.40} & 67.6 \\
{\color{asparagus}KNIFE} & 0.93 & 0.10 & 0.66  & 0.50 & 0.07 & 0.92 & 0.05 & 0.65 & 48.4 \\
{\color{asparagus}\method} & \textbf{0.60} & 0.89 & \underline{0.41} & \underline{0.21} & 0.93 & 0.96 & 1.46 & 0.39 & \textbf{0.09} \\
\bottomrule
\multicolumn{10}{c}{\footnotesize *We exclude InfoNet on this task, as InfoNet cannot output exact MI for data beyond 1D.} \\
\end{tabular}
\end{table*}

\subsection{Results}

\paragraph{High-dimensional independence testing.} \label{paragraph:independence-test}

We first evaluate \method on its ability to accurately discriminate varying levels of statistical dependency between pairs of random variables in high-dimensional settings. Following the setup in~\citep{goldfeld2022k}, we consider various correlation types. For each type, we generate two populations of paired variables: one with no statistical dependence and another with non-trivial dependence. The goal is to evaluate how well the estimated dependency scores separate the two populations. Performance is measured using the area under the precision-recall curve (AUC). All results are the average of 10 independent trials.

As shown in Fig.~\ref{fig:independent-test}, \method consistently exhibits strong test power in assessing statistical dependence in high-dimensional settings, particularly when the sample size exceeds 500 (with the exception of correlation type 2 at dimensionality 128). In these settings, \method achieves performance comparable to gradient-based neural estimators such as MINE, InfoNCE, and MINDE, despite requiring no gradient-based optimization. Compared to InfoNet, \method attains comparable or superior performance in most cases, especially at moderate-to-large sample sizes (e.g., $n \geq 500$).

\vspace{0.125cm}

Overall, these results show that \method can reliably quantify high-dimensional dependence in a zero-shot manner, suggesting its potential for efficient dependence analysis.

\paragraph{Sanity check on benchmark with known MI.} \label{paragraph:bmi}
We next consider the benchmarks proposed in \citep{czyz2023beyond},
where we select 8 representative tasks with analytically derived ground-truth MI.
The test distributions exhibit diverse statistical patterns, ranging from complex Spiral transformation to heavy-tailed Student-$t$ distributions. We generate 5,000 samples for each task.

As presented in Table~\ref{tab:bmi},
the MI values predicted by \method closely align with the ground-truth MI across all tasks.
In all tasks,
\method achieves comparable accuracy to the best neural baselines, while being approximately 300$\times$ faster. Compared to KSG and KDE, which requires no neural network training, \method  offers a clear advantage in accuracy.

\begin{figure*}[t!]
    \centering
    \begin{minipage}{0.277\textwidth}
        \centering
        \includegraphics[width=\textwidth]{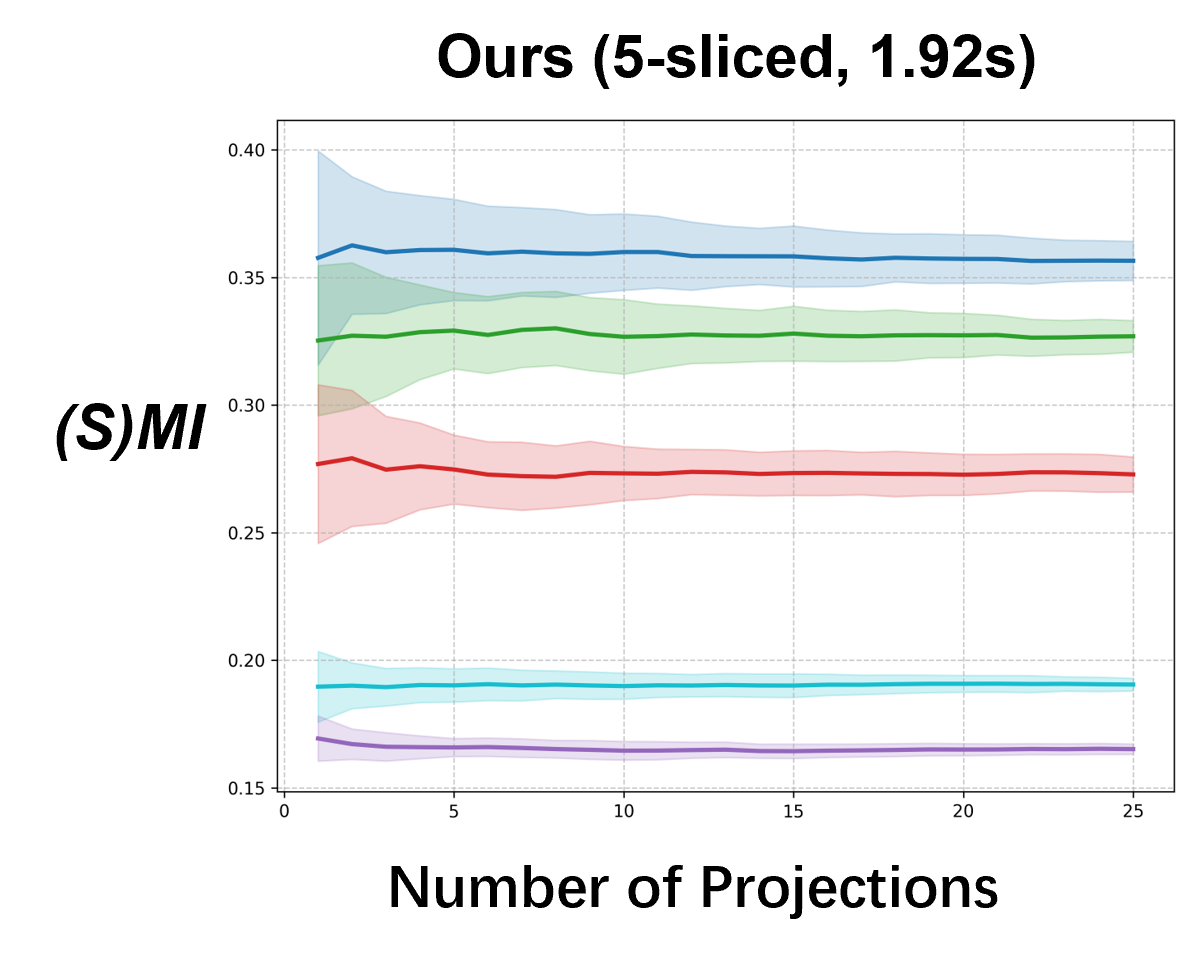}
    \end{minipage}%
    \hfill
    \begin{minipage}{0.228\textwidth}
        \centering
        \includegraphics[width=\linewidth]{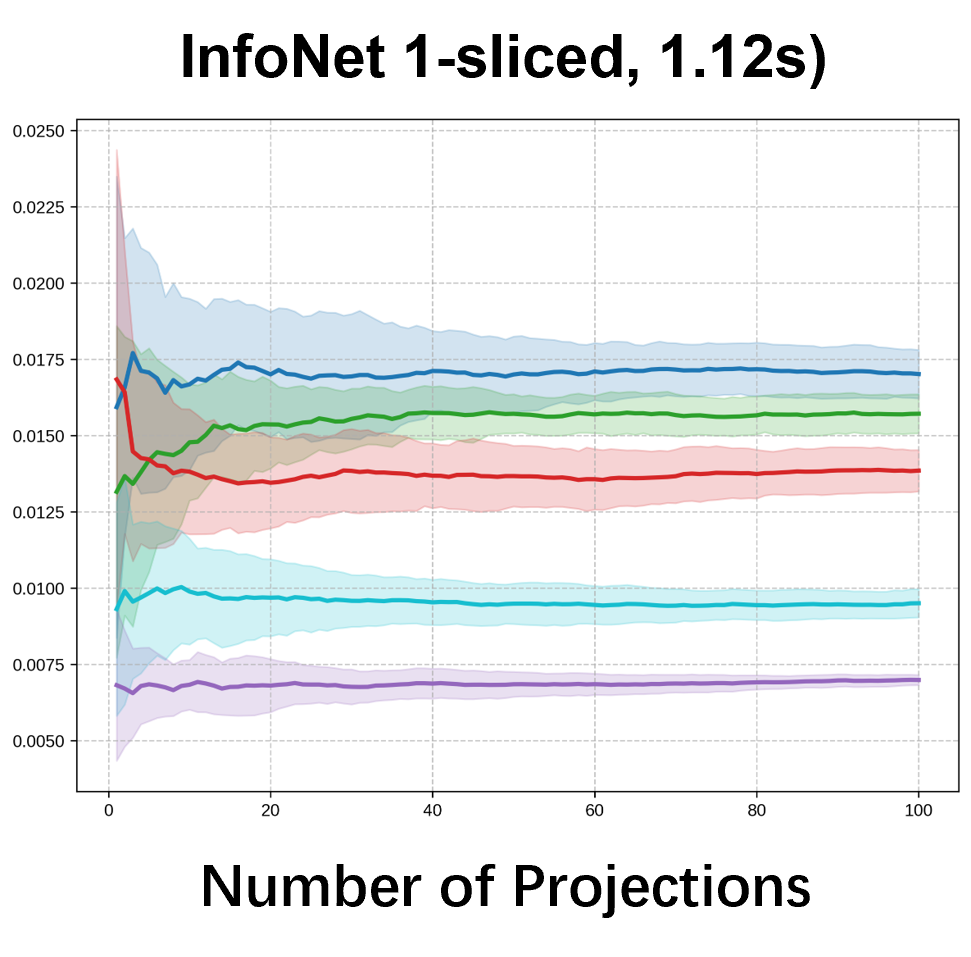}
    \end{minipage}%
    \hfill
    \begin{minipage}{0.228\textwidth}
        \centering
        \includegraphics[width=\linewidth]{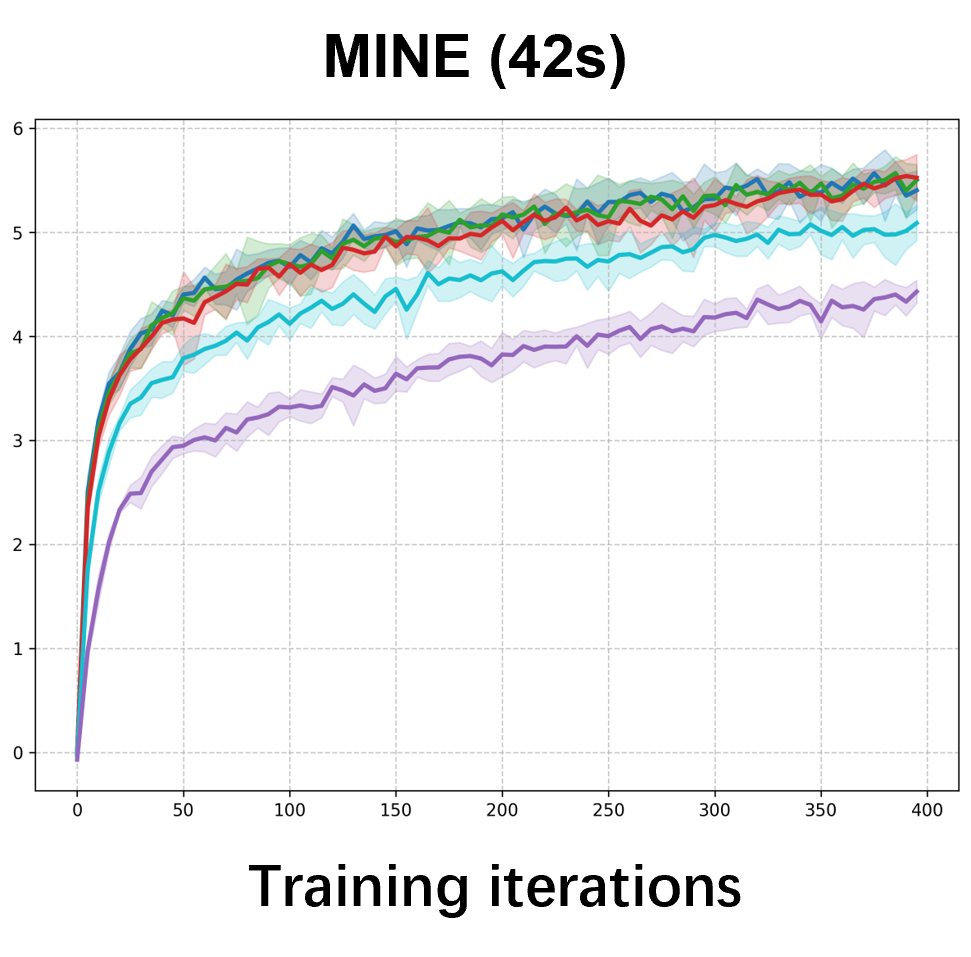}
    \end{minipage}%
    \hfill
    \begin{minipage}{0.228\textwidth}
        \centering
        \includegraphics[width=\linewidth]{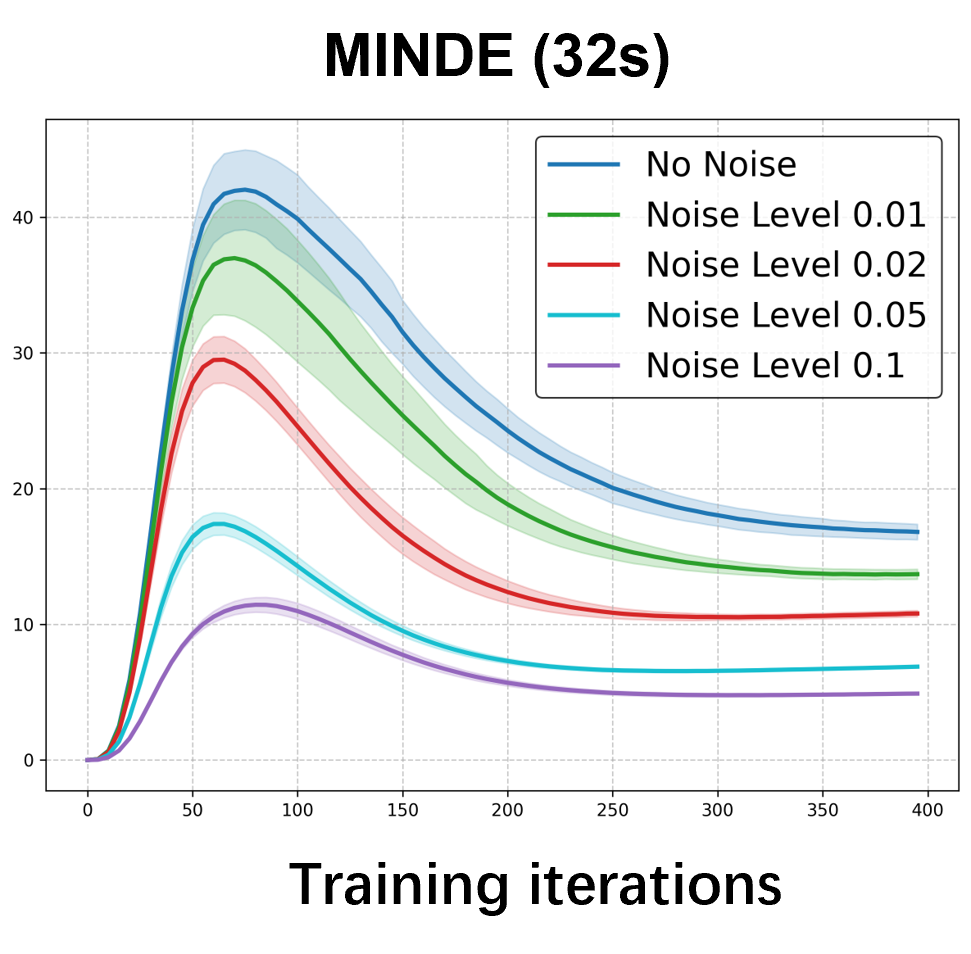}
    \end{minipage}
    \caption{\textbf{Comparing different methods
    on 512-dimensional CLIP-encoded image-text representations
    across five noise levels}.
    The light-colored areas indicate error bounds from 20 repeated experiments.
    \textbf{(Left to right)} \method with 5-sliced MI using $S=25$ random projections;
    InfoNet with 1-sliced MI using more projections (up to $S=128$);
    MINE and MINDE estimating original MI via gradient-based optimization.
    \method demonstrates superior noise level discrimination with clearly separated error bounds, while maintaining significantly faster computation time (noted in parentheses) compared to neural-based alternatives.
    Slicing are not employed in MINE and MINDE, as this will require training $S$ different networks for $S$ slicing directions and does not contribute to efficiency. }
    \label{fig:clip_compare}
\end{figure*}

\paragraph{CLIP-based image-text embedding analysis}
The CLIP model \citep{radford2021learning} encodes images and text into a shared feature space, enabling robust cross-modal understanding by measuring similarity. Here, we assess the correlation between images and their corresponding text annotations by estimating the MI between their latent representations encoded by the pre-trained CLIP model.

We utilize the COCO Captions dataset \citep{chen2015microsoft},
selecting 33,000 image-caption pairs and encoding them into 512-dimensional feature vectors using CLIP.
By systematically introducing Gaussian noises to the data,
we create conditions where statistical dependence naturally decreases. Our objective is to evaluate whether different estimators can effectively detect these changes with high sensitivity -- a spirit similar to the self-consistency test~\citep{song2019understanding}. For each noise level,
we conduct 20 experiments, and we report both the mean and the standard deviation.

Our results in Figure~\ref{fig:clip_compare} demonstrate that \method achieves a strong performance in detecting noise fluctuations, yielding clearly separated error bounds across different noise levels and hence high sensitivity w.r.t the dependence strengths, while being substantially more efficient than alternative approaches.  InfoNet~\citep{hu2024infonet} is the only method comparable in efficiency, but its accuracy is much worse than our method due to its reliance on 1-slicing, which discards a large amount of information despite using more slicing directions.

\paragraph{Real-world motion trajectory modeling}
To assess \method's generalization ability to accurately capture complex real-world relationships, we utilize the PointOdyssey dataset~\citep{zheng2023pointodyssey}, which contains multi-dimensional ground-truth motion trajectories of points on objects across video frames. In this task, a reference point \(P^*\) is selected, and we estimate the mutual information \(\mi(\text{trajectory}(P^*), \text{trajectory}(P))\) between its trajectory and those of all other points \(P\) in the video. Since points on the same object \(O\) typically exhibit stronger spatial correlations in their trajectories than those on different objects, the following relationship is expected (where $t$ is a threshold):
\[
\begin{cases}
\mi(\text{trajectory}(P^*), \text{trajectory}(P)) > t & \text{if } P^*, P \in O \\
\mi(\text{trajectory}(P^*), \text{trajectory}(P)) \leq t  & \text{if }  P^* \in O, P \notin O
\end{cases}
\]
An accurate mutual information estimator should correctly reflect this relationship. Figure~\ref{fig:motion-3subfig} visualizes the raw MI estimates between the reference point and all other points. \method successfully identifies points that belong to the same object by flagging a high MI value, demonstrating its effectiveness in modeling complex spatial dependencies in real-world motion data. Remarkably, \method delivers results for all point pairs within only a few seconds.

\begin{figure*}[t!]
    \centering

    \newcommand{\IMGH}{2.8cm}   
    \newcommand{\SHIFT}{-2.6mm}  

    \begin{subfigure}[t]{0.31\textwidth}
        \centering
        \begin{minipage}[t][\IMGH][t]{\textwidth}
            \centering
            \includegraphics[width=\textwidth]{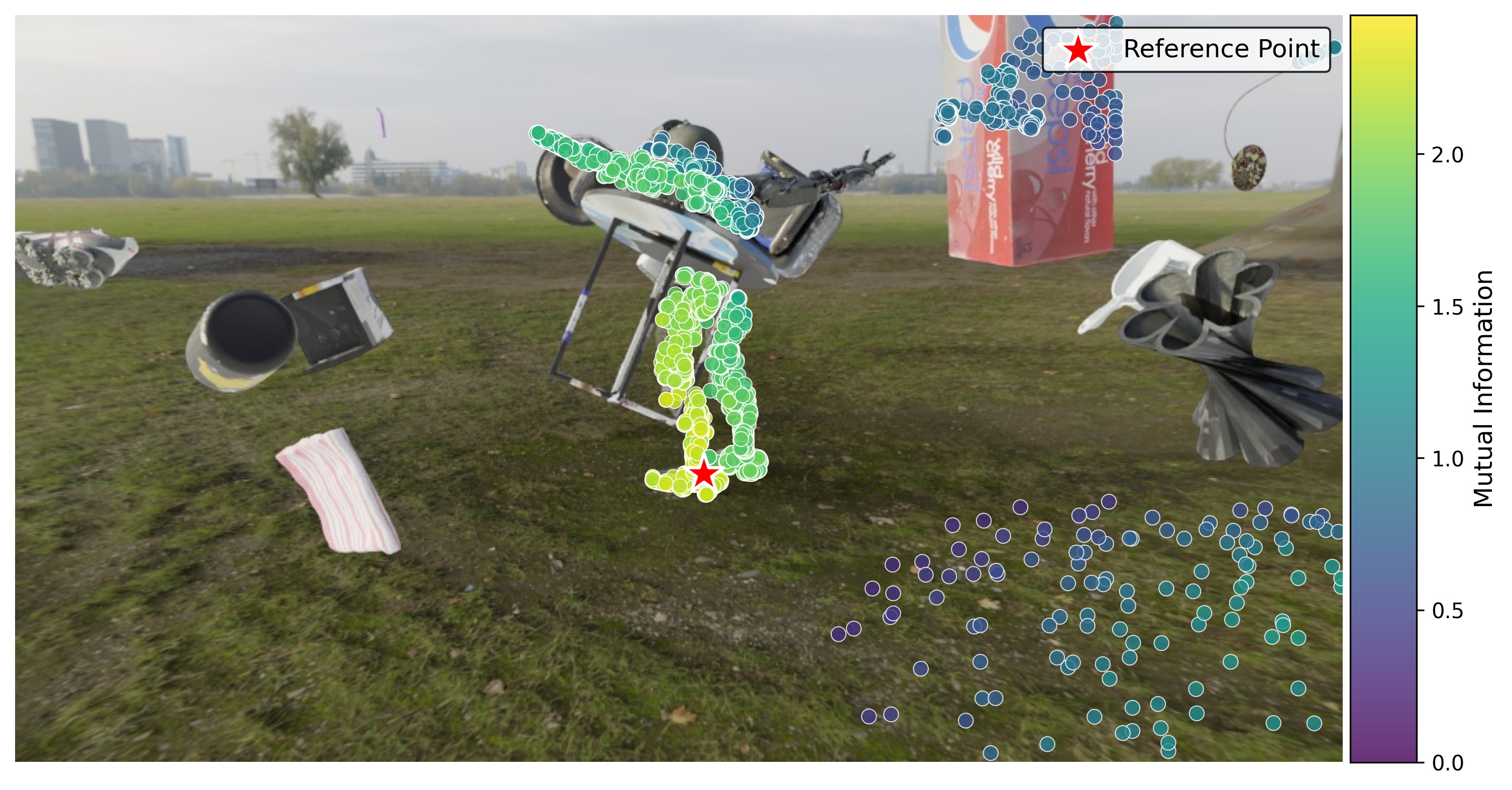}
        \end{minipage}
        \caption{MI heatmap w.r.t.\ reference point 1 ($\star$).}
        \label{fig:video_mi1}
    \end{subfigure}\hfill%
    \begin{subfigure}[t]{0.31\textwidth}
        \centering
        \begin{minipage}[t][\IMGH][t]{\textwidth}
            \centering
            \includegraphics[width=\textwidth]{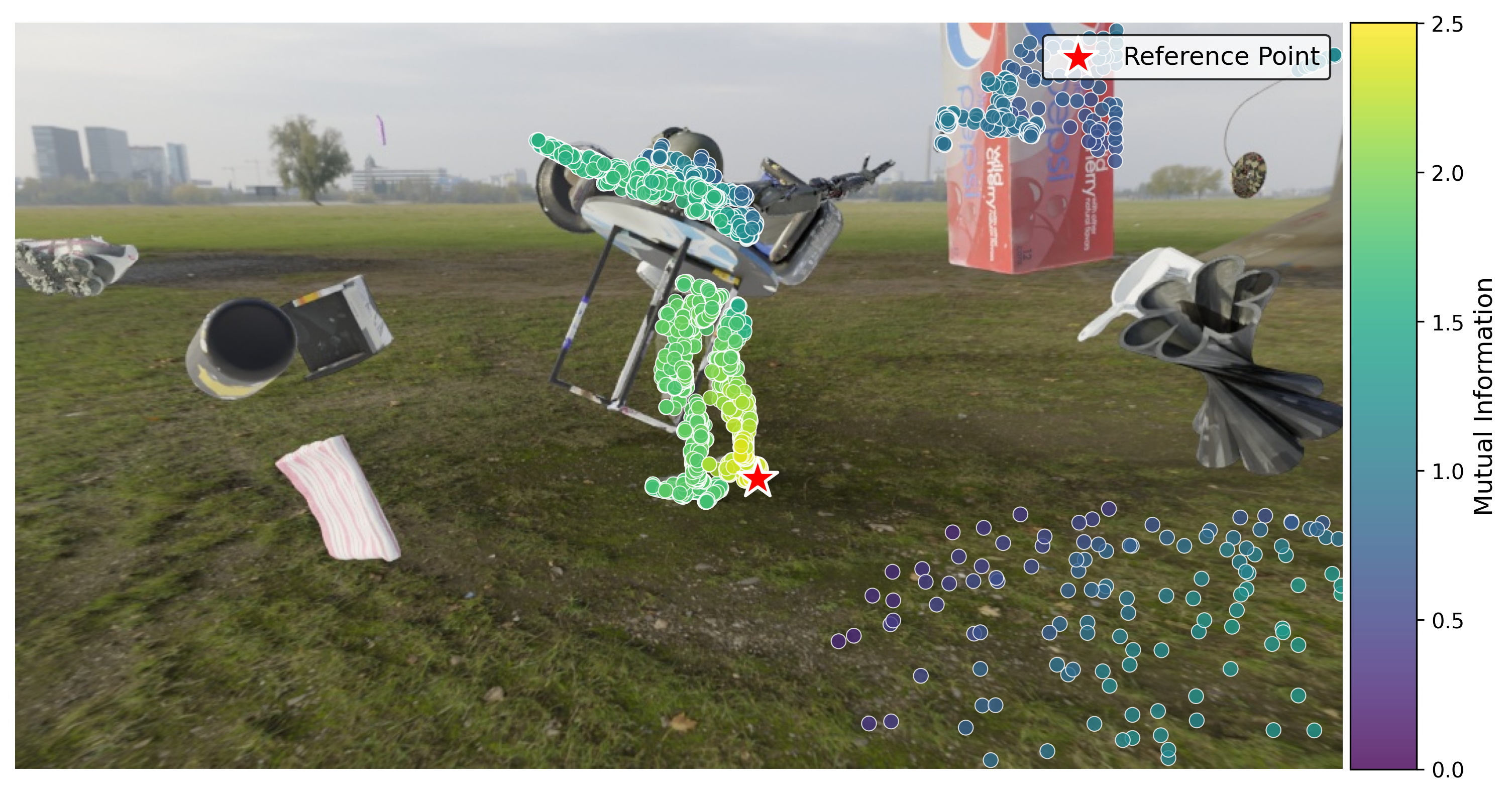}
        \end{minipage}
        \caption{MI heatmap w.r.t.\  reference point 2 ($\star$).}
        \label{fig:video_mi2}
    \end{subfigure}\hfill%
    \begin{subfigure}[t]{0.327\textwidth}
        \centering
        \begin{minipage}[t][2.9cm][t]{\textwidth}
            \centering
            \raisebox{\SHIFT}{\includegraphics[width=\textwidth]{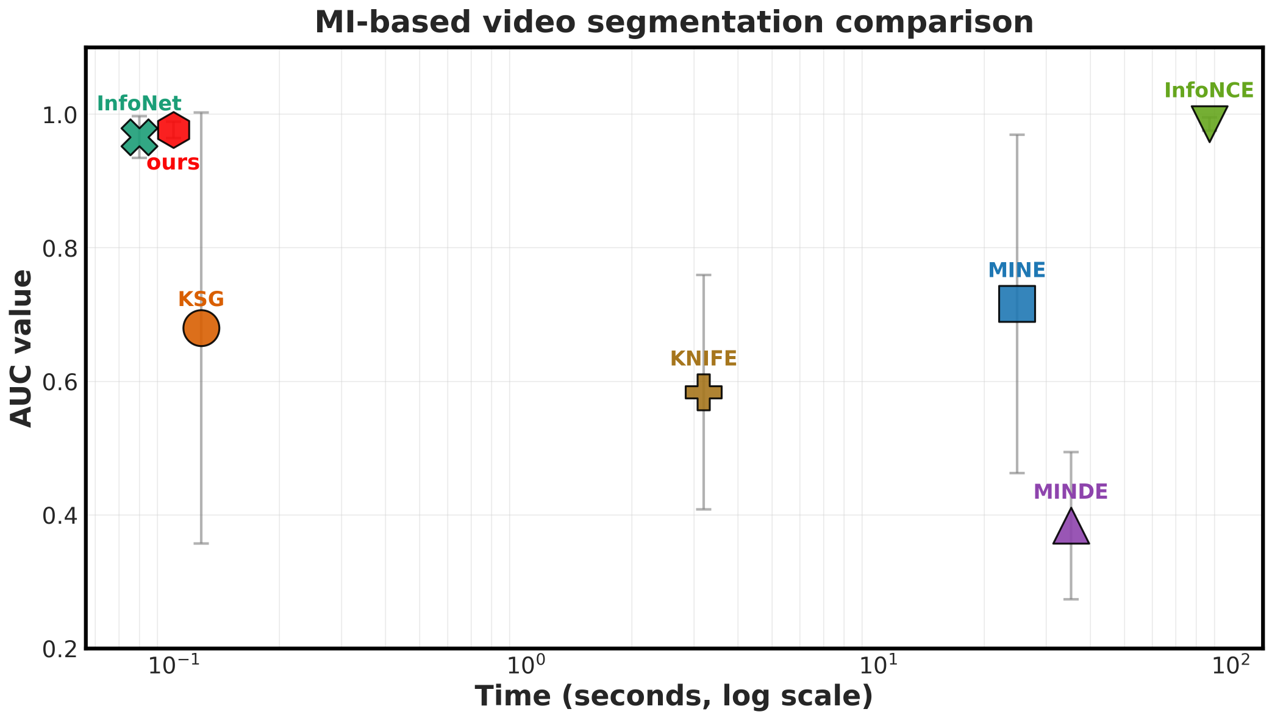}}
        \end{minipage}
        \caption{MI-based video segmentation.}
        \label{fig:video_aucpr}
    \end{subfigure}

    \caption{\textbf{Point trajectory mutual information for video object segmentation on PointOdyssey}~\citep{zheng2023pointodyssey}.
    We estimate mutual information $\mi(\text{trajectory}({P^*}),\text{trajectory}(P))$ between a reference point trajectory $P^*$ (marked by $\star$) and every other point trajectory $P$ across video frames, yielding $\sim 4\times 10^3$ MI terms per video.
    (a,b) The estimated MI is consistently higher for points belonging to the same object as the reference point than for points on other objects.
    (c) Using trajectory MI $\mi(\text{trajectory}({P^*}),\text{trajectory}(P))$ as an affinity for video segmentation, where we report the area under the precision--recall curve (AUC-PR) and the time of different MI estimators.}
    \label{fig:motion-3subfig}
\end{figure*}

\begin{table*}[t!]
\centering
\caption{\textbf{Comparison of policy success rates using key states extracted by different MI estimators across three robotic tasks}. MINE-100 denotes training MINE for 100 iterations, while No-MI-Loss removes the MI maximization term when identifying key states. MI is estimated on 100-dimensional variables using a batch size of 100. \method uses 25 slices, whereas InfoNet uses 250 slices.}
\label{tab:sr-comparison}
\begin{tabular*}{\textwidth}{@{\extracolsep{\fill}} l c c c c c c c}
\toprule
\multirow{2}{*}{\textbf{Tasks}} &
\multicolumn{2}{c}{\textbf{Pick Cube}} &
\multicolumn{2}{c}{\textbf{Stack Cube}} &
\multicolumn{2}{c}{\textbf{Peg Insertion}} &
\multirow{2}{*}{\textbf{Time (s)}} \\
\cmidrule(lr){2-3} \cmidrule(lr){4-5} \cmidrule(lr){6-7}
& \textbf{Seen} & \textbf{Unseen} & \textbf{Seen} & \textbf{Unseen} & \textbf{Seen} & \textbf{Unseen} & \\
\midrule
No-MI-Loss & 66.0 & 60.0 & 67.4 & \textbf{41.0} & 38.6 & 9.3 & -- \\
MINE-100 & 86.4 & 81.0 & 68.0 & 37.0 & 55.0 & 13.5 & 0.62 \\
MINE-1000 & 81.2 & 81.0 & 61.2 & 37.0 & 65.4 & 17.8 & 6.01 \\
InfoNet & 91.0 & 76.0 & 63.0 & 27.0 & 46.4 & 9.8 & 1.23 \\
\method (Ours) & \textbf{94.2} & \textbf{82.0} & \textbf{68.2} & 37.0 & \textbf{72.4} & \textbf{18.3} & 2.17 \\
\bottomrule
\end{tabular*}
\end{table*}

We further evaluate \method on 12 objects sampled from 6 different videos, where we segment video objects using MI estimators.  A visualization of these objects can be found in Fig.~\ref{fig:3d_track_full1} and Fig.~\ref{fig:3d_track_full2} in the appendix. In Fig.~\ref{fig:motion-3subfig}, we compare different methods by comparing the estimated pointwise correlations with the ground-truth segmentation masks. On this out-of-distribution dataset, \method achieves competitive segmentation accuracy while being orders more efficient.

\paragraph{Robotic manipulation concept discovery} \label{paragraph:maxmi}
To fully demonstrate our \method's strong potential in complex, large-scale applications in real world, we further apply our method to a robotic manipulation concept discovery task. The goal is to identify \emph{key states}–critical moments in a trajectory $\tau^i = \{s^i_t\}_{t=1}^T$ that carry strong physical significance (e.g., "peg aligned with hole"). Identifying such key states has been shown to significantly improve robotics policy training~\citep{zhou2024maxmi}. Following recent frameworks~\citep{zhou2024maxmi}, we extract key states by maximizing the mutual information $\mi(s_{t}; s_{t - \Delta t})$ between the key state $s_{t}$ and its prior states $s_{t - \Delta t}$ ($\Delta t > 0$):
\[
    \max \mi(s_{t}; s_{t - \Delta t}), \enskip \forall t \in {1, 2, ..., T}
\]
where each $s_t \in \mathbb{R}^{100}$  encodes the states of an environmental observation. Here $\Delta t = 2$ and $T=200$. We use $5$-sliced MI as a surrogate of $\mi(s_{t}; s_{t - \Delta t})$, where each sliced MI is computed using \method. A total number of $S=40$ slices are used.

We train manipulation policies using key concepts extracted from different MI estimators and compare the success rates (SR) of the resulting policies. We consider three manipulation tasks from ManiSkill 2~\citep{gu2023maniskill2}—Pick Cube, Stack Cube, and Peg Insertion. As summarized in Table~\ref{tab:sr-comparison}, concepts derived from \method substantially outperform those from other methods in terms of success rate under equal or even reduced time budgets. These results highlight that our model generalizes effectively to unseen, complex real-world data, serving as a versatile toolbox for dependence assessment in modern machine learning applications.

\section{Related Work}
\label{sec:related}

\textbf{Neural mutual information estimators}. A series of powerful, neural network-based methods have been developed for MI estimation. The most prominent among these is the MINE estimator~\citep{belghazi2018mutual}, which builds on the Donsker–Varadhan representation. Other approaches rely on density-ratio estimation~\citep{letizia2024mutual, gutmann2010noise}, direct density modeling~\citep{song2019understanding}, score function estimation~\citep{franzese2023minde}, or leverage normalizing flows~\citep{duong2023diffeomorphic, butakov2024mutual} and autoencoders~\citep{gowri2024approximating, butakov2023information} to construct MI estimates. Despite methodological differences, these methods are primarily designed to improve estimation accuracy, and the computational overhead associated with neural network training is often overlooked in real-world deployment. In contrast, we target the orthogonal dimension of computational efficiency, replacing costly iterative optimization with a lightweight forward pass at inference time.

\textbf{Efficient computation of mutual information}.
Various approaches have been developed for rapid MI computation,
each with different trade-offs. Non-parametric methods \citep{kraskov2004estimating, moon1995estimation, silverman2018density} offer training-free efficiency but typically lack the capacity to capture complex dependencies in high-dimensional data. Copula-based approaches \citep{keziou2016semiparametric,safaai2018information,purkayastha2024fastmi,zeng2018jackknife} balance efficiency with accuracy by assuming data follows a known copula family (e.g., Gaussian copula), but this assumption limits their applicability to general distributions.
The recent InfoNet \citep{hu2024infonet} enables fast MI estimation through neural network pretraining -- a concept related to our work. However, InfoNet is restricted to scalar inputs due to its lookup table designs and limited pretraining,
whereas we support handling of multivariate variables with varying dimensionalities using a single unified model. 
Concurrent to our work, MIST~\citep{gritsai2025mist} also considers pretrained MI estimators for multivariate data with varying dimensionalities, tailored to low-sample regimes. However, its supervised pretraining requires ground-truth MI values, restricting the training distributions to those with known MI or those constructed through MI-preserving transformations, whereas our universal dependence generation framework can leverage arbitrary base distributions without requiring their MI to be known.

\vspace{0.22cm}

\textbf{Foundation models for statistical analysis}.
Recent advances in large-scale pretrained models have enabled direct statistical analysis on raw data without gradient-based optimization at test time.
For instance, LLM-based frameworks~\citep{requeima2024llm, siddiqui2025evaluating} directly leverage large language models to perform classification and regression through direct inference, while~\citep{sun2026lambda} develops a LLM-based data agent.
Beyond off-the-shelf LLMs, the community has also developed pretrained transformers tailored to specific data-analysis tasks, including prediction on small tabular datasets~\citep{hollmann2025accurate}, Bayesian inference~\citep{vetter2025effortless, teh2025solving}, and time-series forecasting~\citep{ansari2024chronos}.
Our work is closely related to this emerging paradigm.
However, unlike existing methods primarily focus on \emph{one-way} prediction or forecasting, we study the problem of quantifying \emph{mutual} dependence between two multivariate random variables. This requires modeling interactions between two sets of samples rather than predicting one variable from another, motivating our dual-path attention architecture and diverse synthetic distribution generation strategy for effective generalization.

\section{Conclusion}
\label{sec:discussion}

We introduce \method, a foundation model-like architecture for zero-shot estimation of statistical dependence between multivariate variables. Through large-scale pretraining on synthetic distributions spanning diverse dependence structures and marginal patterns, \method learns to directly infer dependence strengths from raw data in a single forward pass, completely eliminating costly per-dataset optimization. By leveraging slicing techniques, it further scales gracefully and effectively to high-dimensional settings. Extensive evaluations demonstrate that \method matches state-of-the-art neural methods in accuracy while being orders of magnitude faster, and it generalizes effectively to unseen, real-world scenarios. By reframing MI estimation as one-step inference rather than dataset-specific optimization, \method enables a paradigm shift toward scalable dependency estimation, with particular promise for large-scale and real-time applications.

\section*{Acknowledgments}
This work is 
supported by the 
Early Career Scheme 
of the 
Research Grants Council (RGC) 
grant \# 27207224, 
the HKU-100 Award, 
and the HKU Shanghai Intelligent Computing Research Center (ICRC).
Yanzhi Chen acknowledges the support from the Qualcomm Innovation Fellowship.

\section*{Impact Statement}
This paper presents work whose goal is to advance the field of Machine Learning. There are many potential societal consequences of our work, none which we feel must be specifically highlighted here.

\bibliography{main}

@article{goldfeld2021sliced,
  title={Sliced mutual information: A scalable measure of statistical dependence},
  author={Goldfeld, Ziv and Greenewald, Kristjan},
  journal={Advances in Neural Information Processing Systems},
  volume={34},
  pages={17567--17578},
  year={2021}
}

@book{silverman2018density,
  title={Density estimation for statistics and data analysis},
  author={Silverman, Bernard W},
  year={2018},
  publisher={Routledge}
}

@article{vaswani2017attention,
  title={Attention is all you need},
  author={Vaswani, Ashish and Shazeer, Noam and Parmar, Niki and Uszkoreit, Jakob and Jones, Llion and Gomez, Aidan N and Kaiser, {\L}ukasz and Polosukhin, Illia},
  journal={Advances in neural information processing systems},
  volume={30},
  year={2017}
}

@article{moon1995estimation,
  title={Estimation of mutual information using kernel density estimators},
  author={Moon, Young-Il and Rajagopalan, Balaji and Lall, Upmanu},
  journal={Physical Review E},
  volume={52},
  number={3},
  pages={2318},
  year={1995},
  publisher={APS}
}

@article{kraskov2004estimating,
  title={Estimating mutual information},
  author={Kraskov, Alexander and St{\"o}gbauer, Harald and Grassberger, Peter},
  journal={Physical review E},
  volume={69},
  number={6},
  pages={066138},
  year={2004},
  publisher={APS}
}

@article{shannon1948mathematical,
  title={A mathematical theory of communication},
  author={Shannon, Claude Elwood},
  journal={The Bell system technical journal},
  volume={27},
  number={3},
  pages={379--423},
  year={1948},
  publisher={Nokia Bell Labs}
}

@inproceedings{letizia2024mutual,
  title={Mutual Information Estimation via $f$-Divergence and Data Derangements},
  author={Letizia, Nunzio Alexandro and Novello, Nicola and Tonello, Andrea M},
  booktitle={The Thirty-eighth Annual Conference on Neural Information Processing Systems},
  year={2024}
}

@inproceedings{duong2023diffeomorphic,
  title={Diffeomorphic information neural estimation},
  author={Duong, Bao and Nguyen, Thin},
  booktitle={Proceedings of the AAAI Conference on Artificial Intelligence},
  volume={37},
  number={6},
  pages={7468--7475},
  year={2023}
}

@article{tsai2020neural,
  title={Neural methods for point-wise dependency estimation},
  author={Tsai, Yao-Hung Hubert and Zhao, Han and Yamada, Makoto and Morency, Louis-Philippe and Salakhutdinov, Russ R},
  journal={Advances in Neural Information Processing Systems},
  volume={33},
  pages={62--72},
  year={2020}
}

@article{donsker1983asymptotic,
  title={Asymptotic evaluation of certain Markov process expectations for large time. IV},
  author={Donsker, Monroe D and Varadhan, SR Srinivasa},
  journal={Communications on pure and applied mathematics},
  volume={36},
  number={2},
  pages={183--212},
  year={1983},
  publisher={Wiley Online Library}
}

@article{jaegle2021perceiver,
  title={Perceiver io: A general architecture for structured inputs \& outputs},
  author={Jaegle, Andrew and Borgeaud, Sebastian and Alayrac, Jean-Baptiste and Doersch, Carl and Ionescu, Catalin and Ding, David and Koppula, Skanda and Zoran, Daniel and Brock, Andrew and Shelhamer, Evan and others},
  journal={arXiv preprint arXiv:2107.14795},
  year={2021}
}

@book{kullback1997information,
  title={Information theory and statistics},
  author={Kullback, Solomon},
  year={1997},
  publisher={Courier Corporation}
}

@article{chen2016infogan,
  title={Infogan: Interpretable representation learning by information maximizing generative adversarial nets},
  author={Chen, Xi and Duan, Yan and Houthooft, Rein and Schulman, John and Sutskever, Ilya and Abbeel, Pieter},
  journal={Advances in neural information processing systems},
  volume={29},
  year={2016}
}

@inproceedings{poole2019variational,
  title={On variational bounds of mutual information},
  author={Poole, Ben and Ozair, Sherjil and Van Den Oord, Aaron and Alemi, Alex and Tucker, George},
  booktitle={International Conference on Machine Learning},
  pages={5171--5180},
  year={2019},
  organization={PMLR}
}

@inproceedings{mcallester2020formal,
  title={Formal limitations on the measurement of mutual information},
  author={McAllester, David and Stratos, Karl},
  booktitle={International Conference on Artificial Intelligence and Statistics},
  pages={875--884},
  year={2020},
  organization={PMLR}
}

@inproceedings{gutmann2010noise,
  title={Noise-contrastive estimation: A new estimation principle for unnormalized statistical models},
  author={Gutmann, Michael and Hyv{\"a}rinen, Aapo},
  booktitle={Proceedings of the thirteenth international conference on artificial intelligence and statistics},
  pages={297--304},
  year={2010},
  organization={JMLR Workshop and Conference Proceedings}
}

@article{tschannen2019mutual,
  title={On mutual information maximization for representation learning},
  author={Tschannen, Michael and Djolonga, Josip and Rubenstein, Paul K and Gelly, Sylvain and Lucic, Mario},
  journal={arXiv preprint arXiv:1907.13625},
  year={2019}
}

@article{oord2018representation,
  title={Representation learning with contrastive predictive coding},
  author={Oord, Aaron van den and Li, Yazhe and Vinyals, Oriol},
  journal={arXiv preprint arXiv:1807.03748},
  year={2018}
}

@article{zheng2023pointodyssey,
  title={Pointodyssey: A large-scale synthetic dataset for long-term point tracking},
  author={Zheng, Yang and Harley, Adam W and Shen, Bokui and Wetzstein, Gordon and Guibas, Leonidas J},
  journal={arXiv preprint arXiv:2307.15055},
  year={2023}
}

@inproceedings{radford2021learning,
  title={Learning transferable visual models from natural language supervision},
  author={Radford, Alec and Kim, Jong Wook and Hallacy, Chris and Ramesh, Aditya and Goh, Gabriel and Agarwal, Sandhini and Sastry, Girish and Askell, Amanda and Mishkin, Pamela and Clark, Jack and others},
  booktitle={International conference on machine learning},
  pages={8748--8763},
  year={2021},
  organization={PMLR}
}

@article{czyz2023beyond,
  title={Beyond Normal: On the Evaluation of Mutual Information Estimators},
  author={Czy{\.z}, Pawe{\l} and Grabowski, Frederic and Vogt, Julia E and Beerenwinkel, Niko and Marx, Alexander},
  journal={arXiv preprint arXiv:2306.11078},
  year={2023}
}

@article{chen2022scalable,
  title={Scalable Infomin Learning},
  author={Chen, Yanzhi and Li, Yingzhen and Weller, Adrian and others},
  journal={Advances in Neural Information Processing Systems},
  volume={35},
  pages={2226--2239},
  year={2022}
}

@inproceedings{belghazi2018mutual,
  title={Mutual information neural estimation},
  author={Belghazi, Mohamed Ishmael and Baratin, Aristide and Rajeshwar, Sai and Ozair, Sherjil and Bengio, Yoshua and Courville, Aaron and Hjelm, Devon},
  booktitle={International conference on machine learning},
  pages={531--540},
  year={2018},
  organization={PMLR}
}

@article{hu2024infonet,
  title={InfoNet: Neural Estimation of Mutual Information without Test-Time Optimization},
  author={Hu, Zhengyang and Kang, Song and Zeng, Qunsong and Huang, Kaibin and Yang, Yanchao},
  journal={arXiv preprint arXiv:2402.10158},
  year={2024}
}

@article{gritsai2025mist,
  title={MIST: Mutual Information Via Supervised Training},
  author={Gritsai, German and Richards, Megan and M{\'e}loux, Maxime and Cho, Kyunghyun and Peyrard, Maxime},
  journal={arXiv preprint arXiv:2511.18945},
  year={2025}
}

@article{goldfeld2022k,
  title={$ k $-Sliced Mutual Information: A Quantitative Study of Scalability with Dimension},
  author={Goldfeld, Ziv and Greenewald, Kristjan and Nuradha, Theshani and Reeves, Galen},
  journal={Advances in neural information processing systems},
  volume={35},
  pages={15982--15995},
  year={2022}
}

@article{tsur2023max,
  title={Max-sliced mutual information},
  author={Tsur, Dor and Goldfeld, Ziv and Greenewald, Kristjan},
  journal={Advances in neural information processing systems},
  volume={36},
  pages={80338--80351},
  year={2023}
}

@article{gowri2024approximating,
  title={Approximating mutual information of high-dimensional variables using learned representations},
  author={Gowri, Gokul and Lun, Xiaokang and Klein, Allon and Yin, Peng},
  journal={Advances in Neural Information Processing Systems},
  volume={37},
  pages={132843--132875},
  year={2024}
}

@article{chen2020neural,
  title={Neural approximate sufficient statistics for implicit models},
  author={Chen, Yanzhi and Zhang, Dinghuai and Gutmann, Michael and Courville, Aaron and Zhu, Zhanxing},
  journal={arXiv preprint arXiv:2010.10079},
  year={2020}
}

@article{butakov2023information,
  title={Information bottleneck analysis of deep neural networks via lossy compression},
  author={Butakov, Ivan and Tolmachev, Alexander and Malanchuk, Sofia and Neopryatnaya, Anna and Frolov, Alexey and Andreev, Kirill},
  journal={arXiv preprint arXiv:2305.08013},
  year={2023}
}

@inproceedings{chen2023learning,
  title={Is learning summary statistics necessary for likelihood-free inference?},
  author={Chen, Yanzhi and Gutmann, Michael U and Weller, Adrian},
  booktitle={International Conference on Machine Learning},
  pages={4529--4544},
  year={2023},
  organization={PMLR}
}

@inproceedings{butakov2024mutual,
  title={Mutual information estimation via normalizing flows},
  author={Butakov, Ivan and Tolmachev, Alexander and Malanchuk, Sofia and Neopryatnaya, Anna and Frolov, Alexey},
  booktitle={The Thirty-eighth Annual Conference on Neural Information Processing Systems},
  year={2024}
}

@article{franzese2023minde,
  title={MINDE: Mutual information neural diffusion estimation},
  author={Franzese, Giulio and Bounoua, Mustapha and Michiardi, Pietro},
  journal={arXiv preprint arXiv:2310.09031},
  year={2023}
}

@inproceedings{chen2025neural,
  title={Neural Mutual Information Estimation with Vector Copulas},
  author={Chen, Yanzhi and Ou, Zijing and Weller, Adrian and Gutmann, Michael U},
  booktitle={The Thirty-ninth Annual Conference on Neural Information Processing Systems},
  year={2025}
}

@inproceedings{song2019understanding,
  title={Understanding the Limitations of Variational Mutual Information Estimators},
  author={Song, Jiaming and Ermon, Stefano},
  booktitle={International Conference on Learning Representations},
  year={2019}
}

@article{tsur2023neural,
  title={Neural estimation and optimization of directed information over continuous spaces},
  author={Tsur, Dor and Aharoni, Ziv and Goldfeld, Ziv and Permuter, Haim},
  journal={IEEE Transactions on Information Theory},
  volume={69},
  number={8},
  pages={4777--4798},
  year={2023},
  publisher={IEEE}
}

@article{du2024lifestyle,
  title={Lifestyle factors and incident multimorbidity related to chronic disease: a population-based cohort study},
  author={Du, Yihui and de Bock, Geertruida H and Vonk, Judith M and Pham, An Thanh and van der Ende, M Yldau and Snieder, Harold and Smidt, Nynke and Krabbe, Paul FM and Alizadeh, Behrooz Z and Lunter, Gerton and others},
  journal={European Journal of Ageing},
  volume={21},
  number={1},
  pages={37},
  year={2024},
  publisher={Springer}
}

@inproceedings{maanpaa2025dense,
  title={Dense Road Surface Grip Map Prediction from Multimodal Image Data},
  author={Maanp{\"a}{\"a}, Jyri and Pesonen, Julius and Hyyti, Heikki and Melekhov, Iaroslav and Kannala, Juho and Manninen, Petri and Kukko, Antero and Hyypp{\"a}, Juha},
  booktitle={International Conference on Pattern Recognition},
  pages={387--404},
  year={2025},
  organization={Springer}
}

@article{czyz2023properties,
  title={On the properties and estimation of pointwise mutual information profiles},
  author={Czy{\.z}, Pawe{\l} and Grabowski, Frederic and Vogt, Julia E and Beerenwinkel, Niko and Marx, Alexander},
  journal={arXiv preprint arXiv:2310.10240},
  year={2023}
}

@inproceedings{zhou2024maxmi,
  title={MaxMI: A Maximal Mutual Information Criterion for Manipulation Concept Discovery},
  author={Zhou, Pei and Yang, Yanchao},
  booktitle={European Conference on Computer Vision},
  pages={88--105},
  year={2024},
  organization={Springer}
}

@inproceedings{blondel2020fast,
  title={Fast differentiable sorting and ranking},
  author={Blondel, Mathieu and Teboul, Olivier and Berthet, Quentin and Djolonga, Josip},
  booktitle={International Conference on Machine Learning},
  pages={950--959},
  year={2020},
  organization={PMLR}
}

@article{keziou2016semiparametric,
  title={Semiparametric estimation of mutual information and related criteria: Optimal test of independence},
  author={Keziou, Amor and Regnault, Philippe},
  journal={IEEE Transactions on Information Theory},
  volume={63},
  number={1},
  pages={57--71},
  year={2016},
  publisher={IEEE}
}

@article{safaai2018information,
  title={Information estimation using nonparametric copulas},
  author={Safaai, Houman and Onken, Arno and Harvey, Christopher D and Panzeri, Stefano},
  journal={Physical Review E},
  volume={98},
  number={5},
  pages={053302},
  year={2018},
  publisher={APS}
}

@article{purkayastha2024fastmi,
  title={fastMI: A fast and consistent copula-based nonparametric estimator of mutual information},
  author={Purkayastha, Soumik and Song, Peter X-K},
  journal={Journal of Multivariate Analysis},
  volume={201},
  pages={105270},
  year={2024},
  publisher={Elsevier}
}

@article{zeng2018jackknife,
  title={Jackknife approach to the estimation of mutual information},
  author={Zeng, Xianli and Xia, Yingcun and Tong, Howell},
  journal={Proceedings of the National Academy of Sciences},
  volume={115},
  number={40},
  pages={9956--9961},
  year={2018},
  publisher={National Acad Sciences}
}

@article{requeima2024llm,
  title={Llm processes: Numerical predictive distributions conditioned on natural language},
  author={Requeima, James and Bronskill, John and Choi, Dami and Turner, Richard and Duvenaud, David K},
  journal={Advances in Neural Information Processing Systems},
  volume={37},
  pages={109609--109671},
  year={2024}
}

@inproceedings{siddiqui2025evaluating,
  title={On Evaluating LLMs’ Capabilities as Functional Approximators: A Bayesian Evaluation Framework},
  author={Siddiqui, Shoaib Ahmed and Chen, Yanzhi and Heo, Juyeon and Xia, Menglin and Weller, Adrian},
  booktitle={Proceedings of the 31st International Conference on Computational Linguistics},
  pages={5826--5835},
  year={2025}
}

@article{chen2015microsoft,
  title={Microsoft coco captions: Data collection and evaluation server},
  author={Chen, Xinlei and Fang, Hao and Lin, Tsung-Yi and Vedantam, Ramakrishna and Gupta, Saurabh and Doll{\'a}r, Piotr and Zitnick, C Lawrence},
  journal={arXiv preprint arXiv:1504.00325},
  year={2015}
}

@article{comanici2025gemini,
  title={Gemini 2.5: Pushing the frontier with advanced reasoning, multimodality, long context, and next generation agentic capabilities},
  author={Comanici, Gheorghe and Bieber, Eric and Schaekermann, Mike and Pasupat, Ice and Sachdeva, Noveen and Dhillon, Inderjit and Blistein, Marcel and Ram, Ori and Zhang, Dan and Rosen, Evan and others},
  journal={arXiv preprint arXiv:2507.06261},
  year={2025}
}

@article{choi2020regularized,
  title={Regularized mutual information neural estimation},
  author={Choi, Kwanghee and Lee, Siyeong},
  year={2020}
}

@inproceedings{pichler2022differential,
  title={A differential entropy estimator for training neural networks},
  author={Pichler, Georg and Colombo, Pierre Jean A and Boudiaf, Malik and Koliander, G{\"u}nther and Piantanida, Pablo},
  booktitle={International Conference on Machine Learning},
  pages={17691--17715},
  year={2022},
  organization={PMLR}
}

@article{papamakarios2021normalizing,
  title={Normalizing flows for probabilistic modeling and inference},
  author={Papamakarios, George and Nalisnick, Eric and Rezende, Danilo Jimenez and Mohamed, Shakir and Lakshminarayanan, Balaji},
  journal={Journal of Machine Learning Research},
  volume={22},
  number={57},
  pages={1--64},
  year={2021}
}

@article{dinh2016density,
  title={Density estimation using real nvp},
  author={Dinh, Laurent and Sohl-Dickstein, Jascha and Bengio, Samy},
  journal={arXiv preprint arXiv:1605.08803},
  year={2016}
}

@article{gu2023maniskill2,
  title={Maniskill2: A unified benchmark for generalizable manipulation skills},
  author={Gu, Jiayuan and Xiang, Fanbo and Li, Xuanlin and Ling, Zhan and Liu, Xiqiang and Mu, Tongzhou and Tang, Yihe and Tao, Stone and Wei, Xinyue and Yao, Yunchao and others},
  journal={arXiv preprint arXiv:2302.04659},
  year={2023}
}

@article{hollmann2025accurate,
  title={Accurate predictions on small data with a tabular foundation model},
  author={Hollmann, Noah and M{\"u}ller, Samuel and Purucker, Lennart and Krishnakumar, Arjun and K{\"o}rfer, Max and Hoo, Shi Bin and Schirrmeister, Robin Tibor and Hutter, Frank},
  journal={Nature},
  volume={637},
  number={8045},
  pages={319--326},
  year={2025},
  publisher={Nature Publishing Group UK London}
}

@article{ansari2024chronos,
  title={Chronos: Learning the language of time series},
  author={Ansari, Abdul Fatir and Stella, Lorenzo and Turkmen, Caner and Zhang, Xiyuan and Mercado, Pedro and Shen, Huibin and Shchur, Oleksandr and Rangapuram, Syama Sundar and Arango, Sebastian Pineda and Kapoor, Shubham and others},
  journal={arXiv preprint arXiv:2403.07815},
  year={2024}
}

@article{teh2025solving,
  title={Solving empirical bayes via transformers},
  author={Teh, Anzo and Jabbour, Mark and Polyanskiy, Yury},
  journal={arXiv preprint arXiv:2502.09844},
  year={2025}
}

@article{vetter2025effortless,
  title={Effortless, Simulation-Efficient Bayesian Inference using Tabular Foundation Models},
  author={Vetter, Julius and Gloeckler, Manuel and Gedon, Daniel and Macke, Jakob H},
  journal={arXiv preprint arXiv:2504.17660},
  year={2025}
}

@article{sun2026lambda,
  title={Lambda: A large model based data agent},
  author={Sun, Maojun and Han, Ruijian and Jiang, Binyan and Qi, Houduo and Sun, Defeng and Yuan, Yancheng and Huang, Jian},
  journal={Journal of the American Statistical Association},
  volume={121},
  number={553},
  pages={1--13},
  year={2026},
  publisher={Taylor \& Francis}
}
\bibliographystyle{icml2026}

\newpage
\appendix
\onecolumn
\newpage

\section{Technical Appendix}
\label{sec:appendix}

\subsection{Theoretical foundations}
\label{Appendix:theoretical-proof}

\newcommand{\E}{\mathbb{E}}
\newcommand{\rvxD}{\rvx_{\D}}
\newcommand{\rvyD}{\rvy_{\D}}
\newcommand{\pjoint}{\p_{\rvxD,\rvyD}}
\newcommand{\pmar}{\p_{\rvxD} \cdot \p_{\rvyD}}

\begin{proposition}[Consistency of the estimator w.r.t sample size $n$]\label{prop:pop-correct}
Let $\p$ be a probability measure over datasets $\D$.
For each $\D$, let $(\rvxD,\rvyD)\sim \pjoint$ with $\pjoint \ll \pmar$ and $I(\rvxD;\rvyD)<\infty$.
For any admissible critic $\theta \in \Theta$, define
\[
\hat{I}_\theta(\rvxD; \rvyD)
:= \E_{\pjoint}[\theta(\rvxD, \rvyD)]
   - \log \E_{\pmar}\!\big[e^{\theta(\rvxD, \rvyD)}\big].
\]
and let $\hat{I}^n_\theta(\rvxD; \rvyD)$ be its empirical estimate with $n$ sample.

Assume:
\begin{enumerate}[leftmargin=*]
\item[(i)] (\textbf{Per-$\D$ attainment}) For $\p$-a.e.\ $\D$, there exists $\theta^\star_{\D}\in\Theta$ attaining the supremum:
$I(\rvxD;\rvyD)=\sup_{\theta\in\Theta}\hat{I}_\theta(\rvxD;\rvyD)$.
\item[(ii)] (\textbf{Hypernetwork capacity}) The optimal selector $\h:\D\mapsto \Theta$ is within the class of the hypernetwork.
\item[(iii)] (\textbf{Population integrability}) The expectation $\E_{\D}[I(\rvxD, \rvyD)]$ satisfies $\E_{\D}[I(\rvxD, \rvyD)] < \infty$
\end{enumerate}
Define $J(\h):=\E_{\D}\!\big[\hat{I}_{\h(\D)}(\rvxD;\rvyD)\big]$ and let $J^n(\h) :=\E_{\D}\!\big[\hat{I}^n_{\h(\D)}(\rvxD;\rvyD)\big]$ be its finite sample estimate with $n$ samples.
Let $\h^\star\in\arg\max_{\h}J^n(\h)$. Then the estimator
\[
\hat{I}_{\h^\star(\D)}(\rvxD;\rvyD)
\]
\end{proposition}
is a consistent estimate to $I(\rvxD; \rvyD)$ $\p$-a.e.\ $\D$.

\begin{proof}
We begin with the following identity:
\[
    \sup_{\h} J^n(\h) = \sup_{\h} \E[\hat{I}^n_{\h(\D)}(\rvxD;\rvyD)] = \E[\sup_{\theta} \hat{I}^n_{\theta}(\rvxD;\rvyD)]
\]
where the second equality comes from the fact that the hypernetwork $\h$ is a universal selector for $\mathcal{D} \to \Theta$,
so that the supremum for each $\hat{I}^n_{\h(\D)}(\rvxD;\rvyD)$ is reachable.

This suggests that for the optimal $\h^* = \arg\max J^n(\mathcal{H})$, we have
\[
    \h^*(\D) = \sup_{\theta} I_{\theta}^n(\rvxD; \rvy_D)
\]
According to~\citep{belghazi2018mutual}, the estimator $\hat{I}(\rvxD, \rvyD) = \sup_{\theta} \hat{I}^n_{\theta}(\rvxD, \rvyD)$ itself is a consistent estimate of $I(\rvxD, \rvyD)$.
This suggests that for each $\D$ and every $\epsilon > 0$, there exists $n(\D) \in \mathbb{N}$, such that
\[
    \Big |I(\rvxD, \rvyD) - \sup_{\theta} \hat{I}^n_{\theta}(\rvxD, \rvyD) \Big| \leq \epsilon, \quad \forall n \geq n(\D), \enskip a.s.
\]
By taking $n' = \sup_{\D} n(\D)$, substituting $\sup_{\theta} I_{\theta}^n(\rvxD; \rvy_D) = \h^*(\D)$, we have that for every $\epsilon > 0$,
\[
    \Big |I(\rvxD, \rvyD) - \hat{I}^{n'}_{\h^*_{\D}}(\rvxD, \rvyD) \Big| \leq \epsilon, \quad \forall \D, \enskip \forall n \geq n', \enskip a.s.
\]
which completes the proof.
\end{proof}

\begin{proposition}[Positive Definiteness of Generated Covariance Matrix]
\label{prop:positive-def-revised}
The covariance matrix constructed by Algorithm \ref{alg:low_rank_cov} is positive definite almost surely, with controllable condition number through the rank parameter $m$.
\end{proposition}

\begin{proof}
We construct $\boldsymbol{\Sigma} = \mathbf{W}\mathbf{W}^\top + \mathbf{D}$ where $\mathbf{W} \in \mathbb{R}^{d \times m}$ with $W_{ij} \sim \mathcal{N}(0,1)$ and $\mathbf{D} = \text{diag}(d_1, \ldots, d_d)$ with $d_i \sim \text{Uniform}(0,1)$.

For any nonzero $\mathbf{x} \in \mathbb{R}^d$:
\[
\mathbf{x}^\top \boldsymbol{\Sigma} \mathbf{x} = \underbrace{||\mathbf{W}^\top \mathbf{x}||_2^2}_{\geq 0} + \underbrace{\sum_{i=1}^d d_i x_i^2}_{> 0 \text{ a.s.}} > 0
\]

The eigenvalues satisfy $\lambda_{\min}(\boldsymbol{\Sigma}) \geq \min_i d_i > 0$ and $\lambda_{\max}(\boldsymbol{\Sigma}) \leq ||\mathbf{W}||_F^2 + \max_i d_i$. The expected condition number scales as $\mathcal{O}(m)$, allowing control over numerical stability.
\end{proof}

\begin{proposition}[Invariance of MI under Noise Padding]
\label{prop:mi-invariance-extended}
Let $(X, Y)$ be random variables with $X \in \mathbb{R}^{d_x}$, $Y \in \mathbb{R}^{d_y}$. For any independent noise variables $\epsilon_X \perp \epsilon_Y \perp (X,Y)$ of arbitrary dimensions, defining $X' = [X, \epsilon_X]$ and $Y' = [Y, \epsilon_Y]$:
\[
I(X'; Y') = I(X; Y)
\]
This invariance holds for any MI estimator, including the DV representation used in \method.
\end{proposition}

\begin{proof}
Since $\epsilon_X \perp \epsilon_Y \perp (X,Y)$, the joint and marginal densities factor as:
\begin{align}
p(x', y') &= p(x, y) \cdot p(\epsilon_x) \cdot p(\epsilon_y) \\
p(x') &= p(x) \cdot p(\epsilon_x), \quad p(y') = p(y) \cdot p(\epsilon_y)
\end{align}

Therefore, the density ratio is preserved:
\[
\frac{p(x', y')}{p(x')p(y')} = \frac{p(x, y)}{p(x)p(y)}
\]

For the DV representation specifically:
\begin{align}
I(X'; Y') &= \sup_{\theta'} \mathbb{E}_{p(x',y')}[\theta'] - \log \mathbb{E}_{p(x') \otimes p(y')}[e^{\theta'}] \\
&= \sup_{\theta} \mathbb{E}_{p(x,y)}[\theta] - \log \mathbb{E}_{p(x) \otimes p(y)}[e^{\theta}] \\
&= I(X; Y)
\end{align}
where the optimal critic $\theta'^*(x', y') = \theta^*(x, y)$ depends only on the non-noise components.
\end{proof}

\begin{corollary}[$k$-Sliced MI Invariance and Approximation]
\label{cor:k-sliced}
For high-dimensional variables with $d > d_{\max}$, the $k$-sliced MI with padding satisfies:
\begin{enumerate}
\item \textbf{Invariance:} For padded variables $X', Y'$ and random projections $\{P_i\}_{i=1}^k$:
\[
I_{k\text{-sliced}}(X'; Y') = \frac{1}{k}\sum_{i=1}^k I(P_i X'; P_i Y') = I_{k\text{-sliced}}(X; Y)
\]

\item \textbf{Approximation Quality:} Under mild regularity conditions:
\[
|I_{k\text{-sliced}}(X; Y) - I(X; Y)| \leq \frac{C}{\sqrt{k}} \cdot \sqrt{\text{Var}_{P}[I(PX; PY)]}
\]
where $C$ is a universal constant and the variance is over random projections.
\end{enumerate}
\end{corollary}

\begin{proof} We prove the two parts respectively as follows.

\textbf{Part 1:} Follows directly from Proposition \ref{prop:mi-invariance-extended} applied to each projection.

\textbf{Part 2:} By the central limit theorem over independent projections:
\[
\sqrt{k}(I_{k\text{-sliced}} - \mathbb{E}_P[I(PX; PY)]) \xrightarrow{d} \mathcal{N}(0, \text{Var}_P[I(PX; PY)])
\]
The bias $|\mathbb{E}_P[I(PX; PY)] - I(X; Y)|$ depends on the projection dimension and decreases as more projections capture the dependency structure.
\end{proof}

\begin{algorithm}[!t]
\caption{Full Training Sequence Generation Pipeline}
\label{alg:gmm_data_gen}
\begin{algorithmic}[1]
\REQUIRE Variable dimensions $d_{\rvx}$, $d_{\rvy}$, max components $K_{\max}=60$, samples $N$, max dim $d_{\max}$
\STATE Randomly select $K \in \{1, 2, \dots, K_{\max}\}$ and sample weights $\{\pi_i\}_{i=1}^K$ s.t. $\sum_{i=1}^K \pi_i = 1$
\STATE Set total dimension $d = d_{\rvx}+d_{\rvy}$
\FOR{each component $i = 1$ to $K$}
    \STATE Sample mean $\boldsymbol{\mu}_i \in \mathbb{R}^{d}$ with elements from $\text{Uniform}([-5, 5])$
    \STATE Select rank $m \in \{1, 2, \dots, d\}$ and generate $\mathbf{W} \in \mathbb{R}^{d \times m}$ with $W_{ij} \sim \mathcal{N}(0, 1)$
    \STATE Generate diagonal matrix $\mathbf{D}$ with $D_{ii} \sim \text{Uniform}(0, 1)$
    \STATE Compute $\boldsymbol{\Sigma}_i' = \mathbf{W}\mathbf{W}^T + \mathbf{D}$ and normalize to correlation matrix $\boldsymbol{\text{Corr}}_i$: $\text{Corr}_{i_{jk}} = \frac{\Sigma'_{i_{jk}}}{\sqrt{\Sigma'_{i_{jj}}\Sigma'_{i_{kk}}}}$
    \STATE Sample variances $\boldsymbol{\sigma}^2 \in \mathbb{R}^d$ from $\text{Uniform}([0.01, 10])$
    \STATE Set $\boldsymbol{\Sigma}_i = \text{diag}(\boldsymbol{\sigma}) \cdot \boldsymbol{\text{Corr}}_i \cdot \text{diag}(\boldsymbol{\sigma})$
\ENDFOR
\STATE Define GMM: $p(\mathbf{z}) = \sum_{i=1}^K \pi_i \mathcal{N}(\mathbf{z} | \boldsymbol{\mu}_i, \boldsymbol{\Sigma}_i)$
\STATE Sample $\mathbf{Z} = \{z^1, z^2, ..., z^N\} \sim p(\mathbf{z})$ and partition each $z^j$ into $x^j \in \mathbb{R}^{d_{\rvx}}$ and $y^j \in \mathbb{R}^{d_{\rvy}}$
\STATE Organize as sequences $\mathbf{X} = \{x^1, ..., x^N\}$ and $\mathbf{Y} = \{y^1, ..., y^N\}$
\STATE If needed, pad $\mathbf{X}$ and $\mathbf{Y}$ with $\mathcal{N}(0, 1)$ noise to dimensions $d_{\max}$
\STATE \textbf{return} $(\mathbf{X}, \mathbf{Y})$
\end{algorithmic}
\end{algorithm}

\subsection{Additional details of synthetic data generation} \label{Appendix:Syn data gen}

\subsection*{Full data generation pipeline}

The following outlines the detailed procedure used to generate synthetic data in our experiments.

\begin{itemize}[leftmargin=*]
    \item \emph{Mixture components number} $K$. We uniformly sample $K$ from the set $\{1,\cdots, 60\}$.
    \item \emph{Weights} $\pi_i$. We uniformly sample each $\pi_i$ from the interval $[0, 1]$, then set $\pi_i \leftarrow \pi_i / \sum^K_{j=1} \pi_j$.
    \item \emph{Mean vector} $\boldsymbol{\mu}_i$. Each element in the mean vector is uniformly sampled from the interval $[-5, 5]$.
   \item \emph{Correlation matrices}. For the correlation matrices in Gaussian copulas and $t$-copula, we introduce a novel low-rank factorization method for covariance matrix construction that ensures meaningful inter-dimensional correlations. We construct the covariance matrix as $\boldsymbol{\Sigma} = \mathbf{W}\mathbf{W}^T + \mathbf{D}$, where $\mathbf{W} \in \mathbb{R}^{d \times m}$ with rank $m \sim \text{Uniform}(\{1, 2, ..., d\})$ has entries $W_{ij} \sim \mathcal{N}(0, 1)$, and $\mathbf{D} = \text{diag}(d_1, ..., d_d)$ with $d_i \sim \text{Uniform}(0, 1)$ ensures positive definiteness (Proposition~\ref{prop:positive-def-revised}). This formulation guarantees that the expected absolute correlation between off-diagonal entries scales as $\mathbb{E}[|\rho_{ij}|] \approx \sqrt{m/(m + 0.5)}$ for $i \neq j$, producing stronger correlations when $m$ is small. By controlling the rank parameter $m$, we systematically vary correlation strength from weak (high rank) to strong (low rank), ensuring the hypernetwork encounters the full spectrum of correlation patterns during training.
   \item \emph{Degree of freedom} $\nu$. For student-$t$ copula, we randomly sample the degree of freedom $\mu$ as $\nu \sim \text{Uniform}([2, 30])$ to vary tail behavior. This exposes the hypernetwork to both short and heavy-tailed dependences.
\end{itemize}

\subsection*{Advantages of the proposed covariance matrix generation mechanism}
\label{Appendix:large_mix_dim}

In this section, we highlight the advantages of our proposed covariance matrix generation mechanism by comparing it to several commonly used alternatives.

We consider three baseline approaches:
\begin{itemize}[leftmargin=*]
    \item \textbf{Full-rank matrix reparameterization}, where the covariance matrix is constructed as \(\mathbf{C} = \mathbf{A}\mathbf{A}^T\), with \(\mathbf{A} \in \mathbb{R}^{d \times d}\) being a full-rank matrix whose entries are sampled independently from \(\mathcal{N}(0,1)\).

    \item \textbf{Cholesky decomposition}, where \(\mathbf{C} = \mathbf{L}\mathbf{L}^T\), and \(\mathbf{L} \in \mathbb{R}^{d \times d}\) is a lower triangular matrix with positive diagonal elements. The diagonal entries of \(\mathbf{L}\) are sampled from \(\mathrm{Uniform}(0,1)\), and the off-diagonal entries from \(\mathcal{N}(0,1)\).

    \item \textbf{Eigenvalue decomposition}, where \(\mathbf{C} = \mathbf{Q}\mathbf{D}\mathbf{Q}^T\), with \(\mathbf{D} \in \mathbb{R}^{d \times d}\) being a diagonal matrix with positive entries sampled from \(\mathrm{Uniform}[0.1, 10.1)\), and \(\mathbf{Q} \in \mathbb{R}^{d \times d}\) being an orthogonal matrix obtained via QR decomposition of a random matrix with entries from \(\mathcal{N}(0,1)\).
\end{itemize}

While these methods guarantee positive definiteness, they often produce covariance matrices with relatively small off-diagonal entries compared to the diagonal, resulting in limited diversity in the induced dependence structure; see the lower panel in Fig.~\ref{fig:correlation_matrix}.

In contrast, our method employs a \emph{low-rank factorization} strategy (see Algorithm~\ref{alg:low_rank_cov}). By tuning the rank parameter \(m \leq d\), we can flexibly control the strength of off-diagonal entries, thereby enabling the generation of covariance matrices with highly diverse dependence structures --- an important design for ensuring training data diversity. This effect is illustrated in the upper panel of Fig.~\ref{fig:correlation_matrix}.

\begin{algorithm}[!t]
\caption{Low-Rank Factorization Method for Covariance Matrix Generation}
\label{alg:low_rank_cov}
\begin{algorithmic}[1]
\REQUIRE Target dimension $d \in \mathbb{N}^+$
\ENSURE Positive definite covariance matrix $\boldsymbol{\Sigma} \in \mathbb{R}^{d \times d}$
\STATE Sample rank parameter $m \sim \text{Uniform}(\{1, 2, \ldots, d\})$
\STATE Generate factor matrix $\mathbf{W} \in \mathbb{R}^{d \times m}$ with entries $W_{ij} \stackrel{\text{i.i.d.}}{\sim} \mathcal{N}(0, 1)$
\STATE Generate diagonal matrix $\mathbf{D} = \text{diag}(d_1, \ldots, d_d)$ with $d_i \stackrel{\text{i.i.d.}}{\sim} \text{Uniform}(0, 1)$
\STATE Compute covariance matrix $\boldsymbol{\Sigma} = \mathbf{W}\mathbf{W}^T + \mathbf{D}$
\STATE \textbf{Optional:} Convert to correlation matrix $\mathbf{R}$ with entries:
\begin{equation*}
R_{ij} = \frac{\Sigma_{ij}}{\sqrt{\Sigma_{ii} \Sigma_{jj}}}
\end{equation*}
\STATE \textbf{Optional:} Rescale to final covariance $\boldsymbol{\Sigma}_{\text{final}} = \text{diag}(\boldsymbol{\sigma}) \mathbf{R} \text{diag}(\boldsymbol{\sigma})$, $\sigma_i \sim \text{Uniform}(0.1, \sqrt{10})$
\STATE \textbf{return} $\boldsymbol{\Sigma}$ (or $\boldsymbol{\Sigma}_{\text{final}}$ if rescaled)
\end{algorithmic}
\end{algorithm}

\begin{figure}[!t]
\centering
\begin{subfigure}{0.32\textwidth}
    \includegraphics[width=\textwidth]{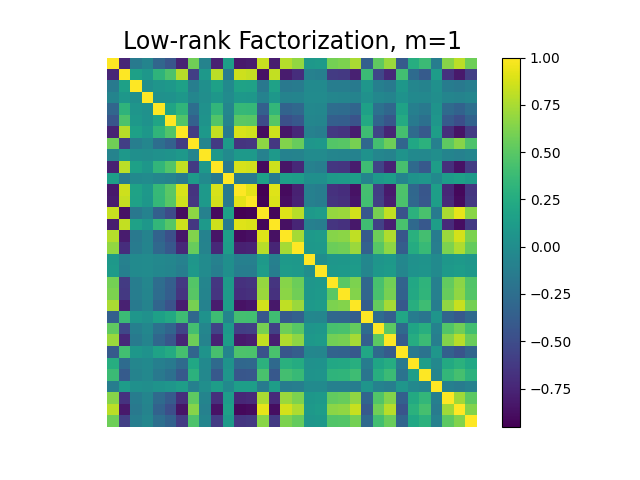}
\end{subfigure}
\hspace{-0.1cm}
\begin{subfigure}{0.32\textwidth}
    \includegraphics[width=\textwidth]{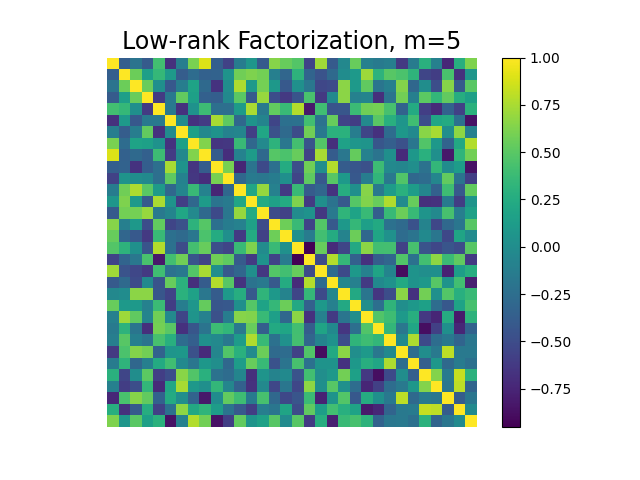}
\end{subfigure}
\hspace{-0.1cm}
\begin{subfigure}{0.32\textwidth}
    \includegraphics[width=\textwidth]{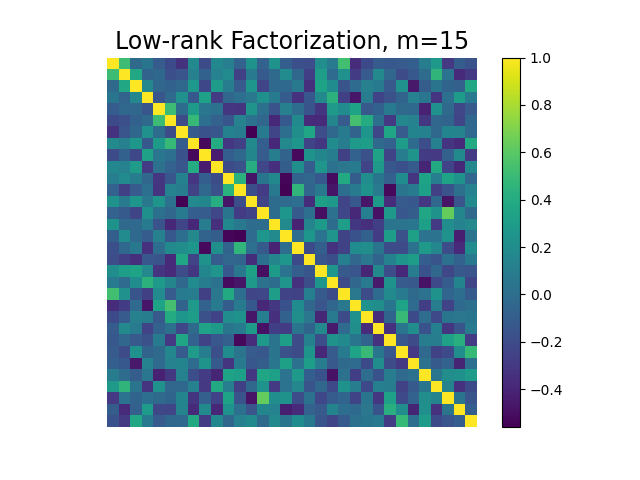}
\end{subfigure}

\vspace{0.2cm}

\begin{subfigure}{0.32\textwidth}
    \includegraphics[width=\textwidth]{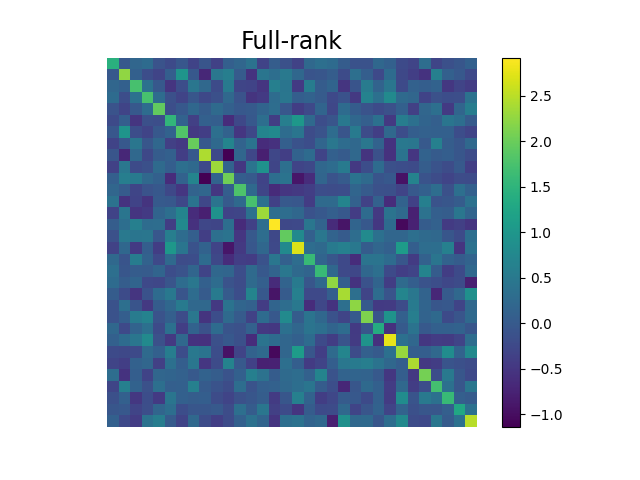}
\end{subfigure}
\hspace{-0.1cm}
\begin{subfigure}{0.32\textwidth}
    \includegraphics[width=\textwidth]{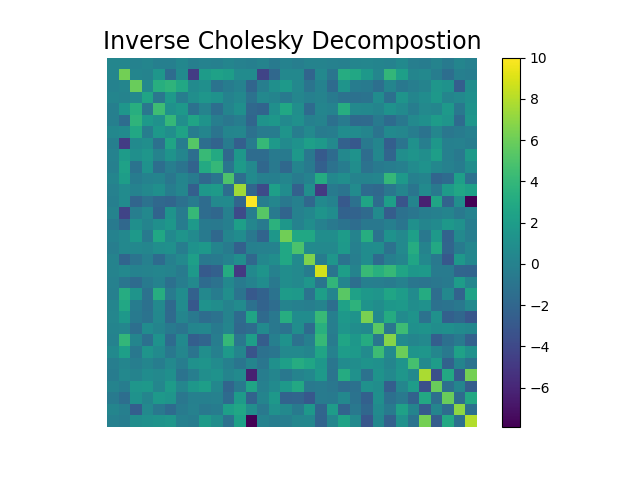}
\end{subfigure}
\hspace{-0.1cm}
\begin{subfigure}{0.32\textwidth}
    \includegraphics[width=\textwidth]{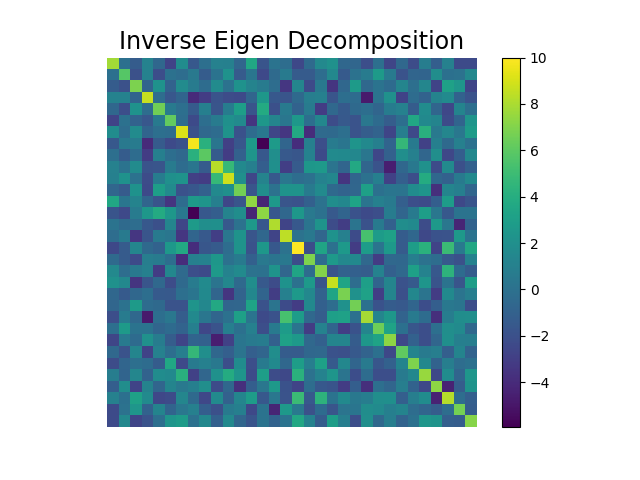}
\end{subfigure}

\caption{Visualization of correlation matrices generated by various methods. Existing approaches often yield small off-diagonal elements, whereas the low rank factor method adjusts their magnitude by tuning the rank factor \(m\).}

\label{fig:correlation_matrix}
\end{figure}

\subsection{Additional experimental details of independent testing experiments} \label{Appendix:detail-express-independent}


\emph{Test cases details}. Below are three different relationships between $X$ and $Y$ in high dimensional independence test in sec.~\ref{paragraph:independence-test}.

(a) \textbf{One feature (linear)}: $X, Z \sim \mathcal{N}\left(0, \mathrm{I}_d\right)$ i.i.d. and $Y=\frac{1}{\sqrt{2}}\left(\frac{1}{\sqrt{d}}\left(\mathbf{1}^{\top} X\right) \mathbf{1}+Z\right)$, where $\mathbf{1}:=$ $(1, \ldots ,1)^{\top} \in \mathbb{R}^d$.

(b) \textbf{Two features}: $X, Z \sim \mathcal{N}\left(0, \mathrm{I}_d\right)$ i.i.d. and $Y_i=\frac{1}{\sqrt{2}} \begin{cases}\frac{1}{d}\left(\mathbf{1}_{\lfloor d / 2\rfloor} 0 \ldots 0\right)^{\top} X+Z_i, & i \leq \frac{d}{2} \\ \frac{1}{d}\left(0 \ldots 0 \mathbf{1}_{\lceil d / 2\rceil}\right)^{\top} X+Z_i, & i>\frac{d}{2} \text {. }\end{cases}$

(c) \textbf{Independent coordinates}: $X, Z \sim \mathcal{N}\left(0, \mathrm{I}_d\right)$ i.i.d. and $Y=\frac{1}{\sqrt{2}}(X+Z)$.

\emph{Dimensionality trends}. In Fig.~\ref{fig:independent-test-9fig}, we examine the impact of dimensionality on estimation performance by considering three settings with increasing dimensions: 16, 64, and 128. As expected, the test power of our method decreases as dimensionality grows, particularly in small-sample regimes ($n \leq 400$).

\begin{figure}[!t]
    \centering

    \begin{subfigure}{.32\textwidth}
        \centering
        \includegraphics[width=\textwidth]{figs/independence-task1-dim16.png}
        \vspace{-2mm}
    \end{subfigure}
    \begin{subfigure}{.32\textwidth}
        \centering
        \includegraphics[width=\textwidth]{figs/independence-task2-dim16.png}
        \vspace{-2mm}
    \end{subfigure}
    \begin{subfigure}{.32\textwidth}
        \centering
        \includegraphics[width=\textwidth]{figs/independence-task3-dim16.png}
        \vspace{-2mm}
    \end{subfigure}

    \begin{subfigure}{.32\textwidth}
        \centering
        \includegraphics[width=\textwidth]{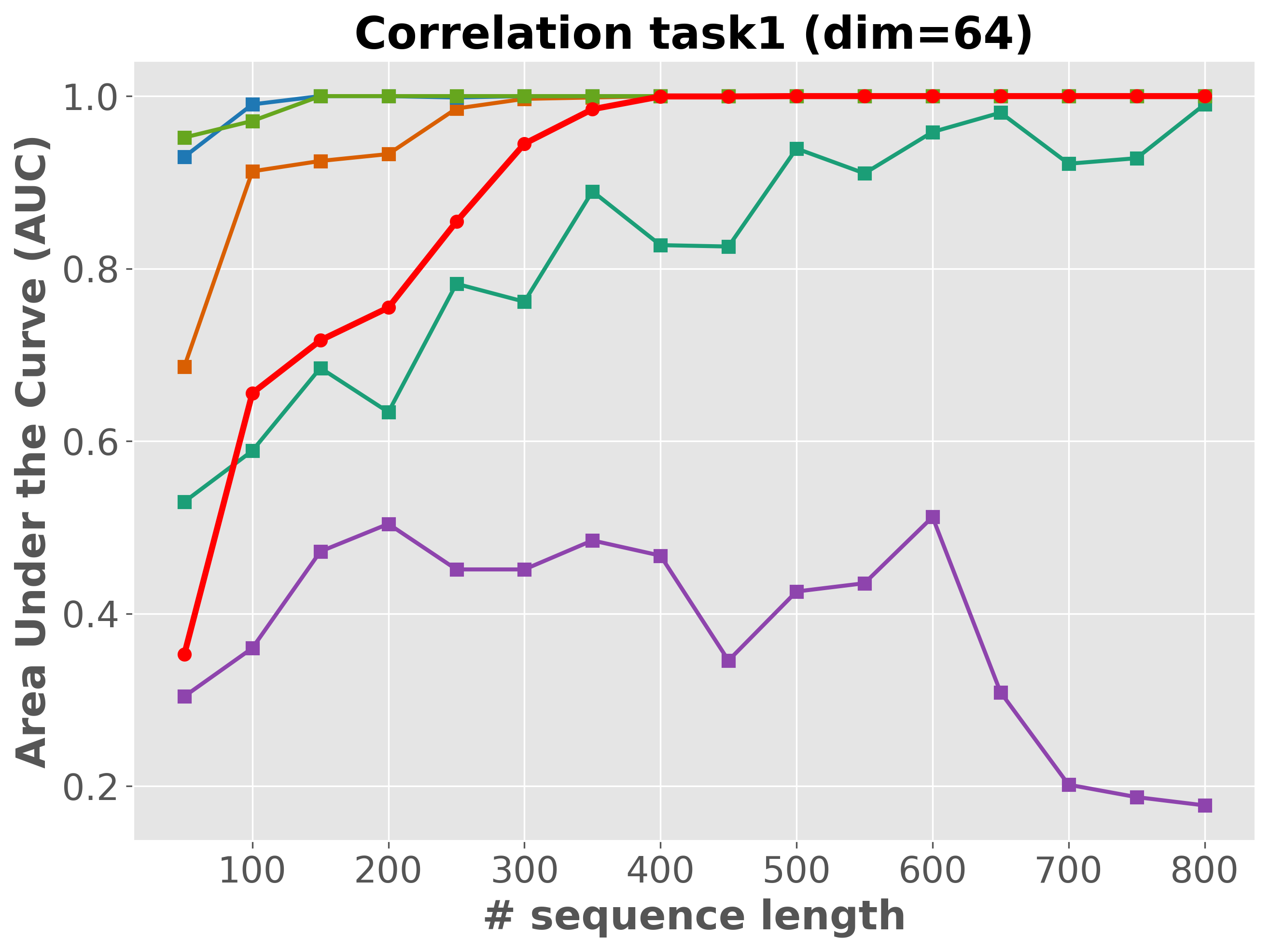}
        \vspace{-2mm}
    \end{subfigure}
    \begin{subfigure}{.32\textwidth}
        \centering
        \includegraphics[width=\textwidth]{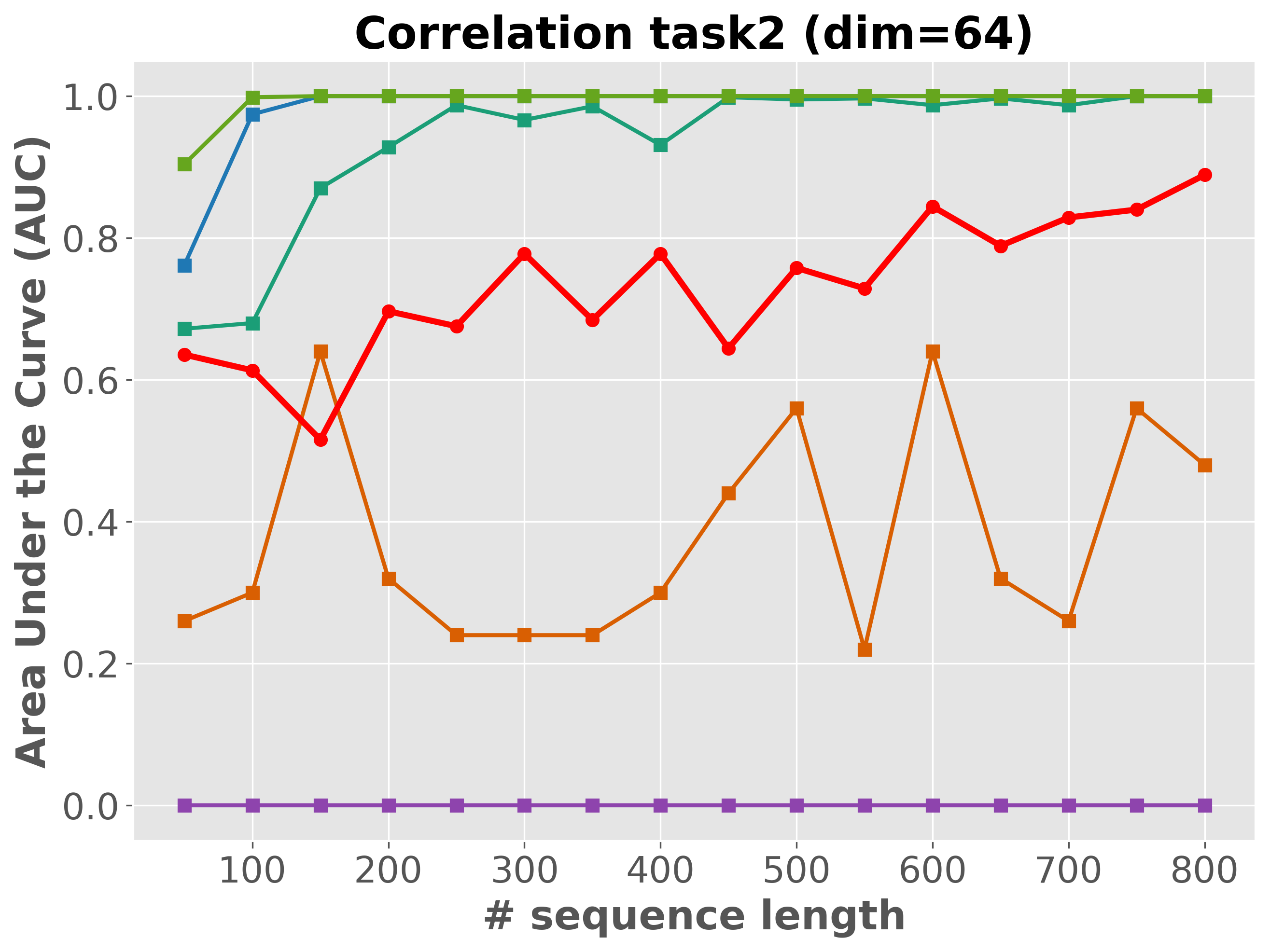}
        \vspace{-2mm}
    \end{subfigure}
    \begin{subfigure}{.32\textwidth}
        \centering
        \includegraphics[width=\textwidth]{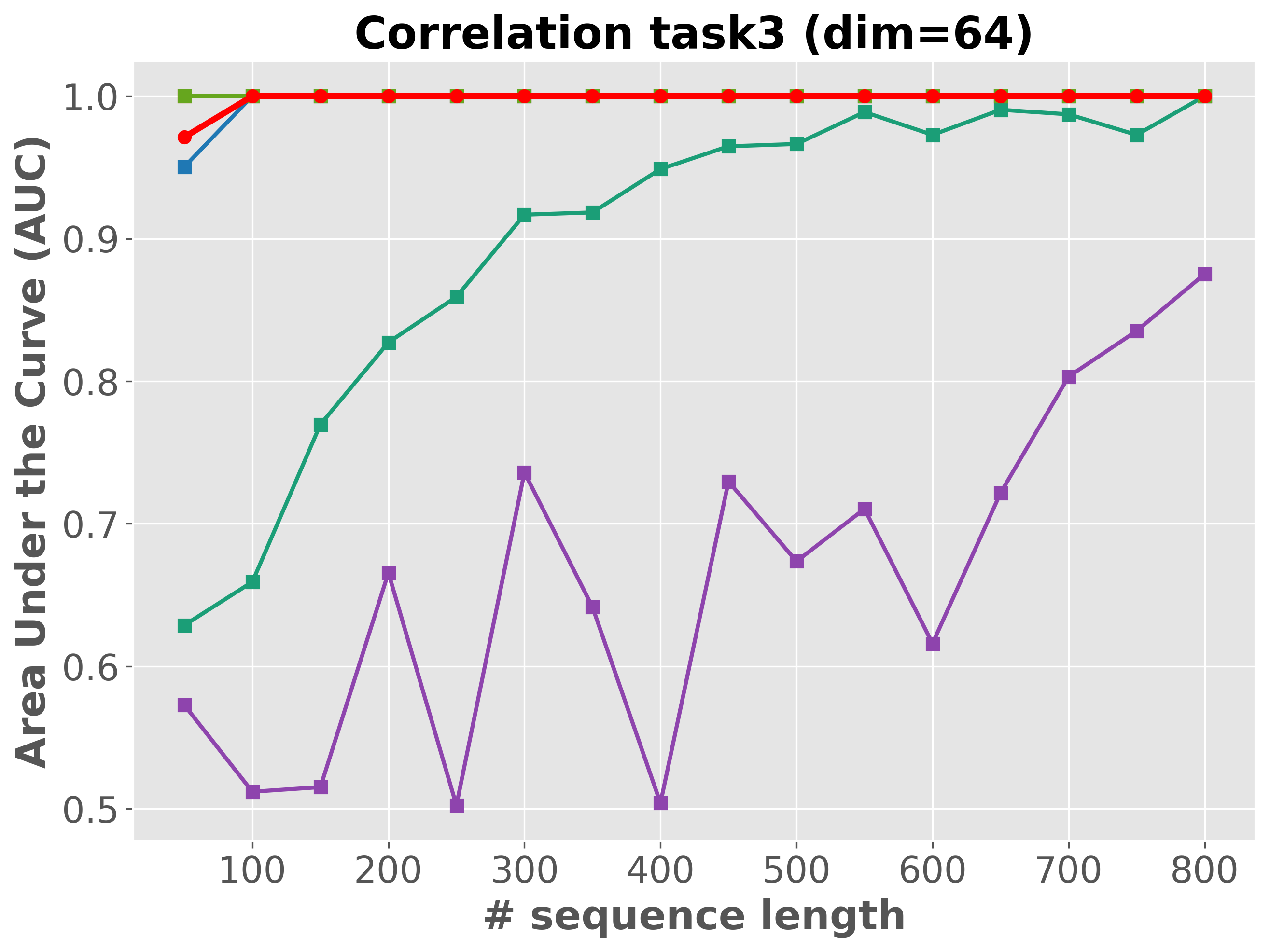}
        \vspace{-2mm}
    \end{subfigure}

    \begin{subfigure}{.32\textwidth}
        \centering
        \includegraphics[width=\textwidth]{figs/independence-task1-dim128.png}
        \vspace{-2mm}
    \end{subfigure}
    \begin{subfigure}{.32\textwidth}
        \centering
        \includegraphics[width=\textwidth]{figs/independence-task2-dim128.png}
        \vspace{-2mm}
    \end{subfigure}
    \begin{subfigure}{.32\textwidth}
        \centering
        \includegraphics[width=\textwidth]{figs/independence-task3-dim128.png}
        \vspace{-2mm}
    \end{subfigure}

    \caption{Independence testing across three correlation types and dimensions (16, 64, 128) across seven methods. Each curve plots the ROC-AUC as a function of sequence length \( n \). The figure demonstrates that performance degrades with increasing dimensionality.}
    \label{fig:independent-test-9fig}
\end{figure}

\subsection{Full results of validation on out-of-domain motion data}\label{app:full-realworld-result}

In this section, we provide detailed results of the experiments on motion data. We only consider points that appear throughout the entire video. Fig.~\ref{fig:3d_track_full1} and Fig.~\ref{fig:3d_track_full2} show the full visualization results of estimated mutual information between one selected point and other points in the videos.

\begin{figure}[t!]
     \centering
     \label{fig:full_3dtrack_1}
     \includegraphics[width=0.96\textwidth]{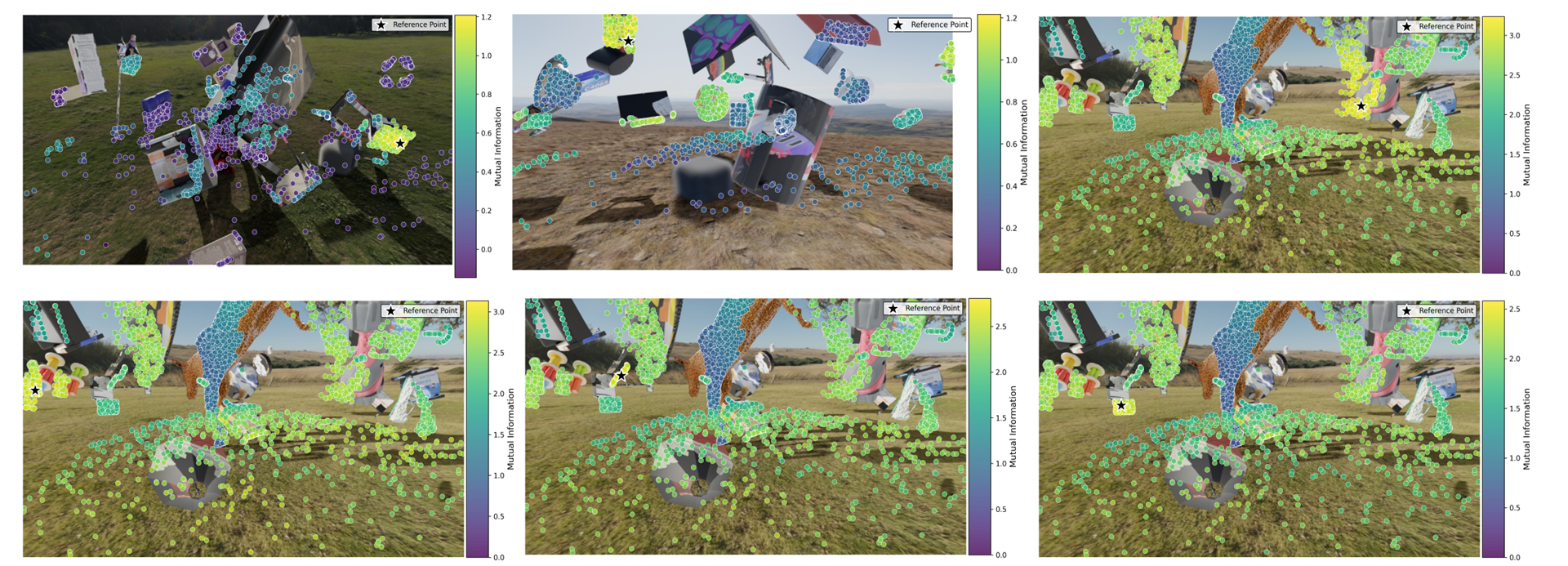}
     \caption{Full visualization results of \method estimated motion data.}
     \label{fig:3d_track_full1}
     \vspace{-8pt}
\end{figure}

\begin{figure}[t!]
     \centering
     \label{fig:full_3dtrack_2}
     \includegraphics[width=0.96\textwidth]{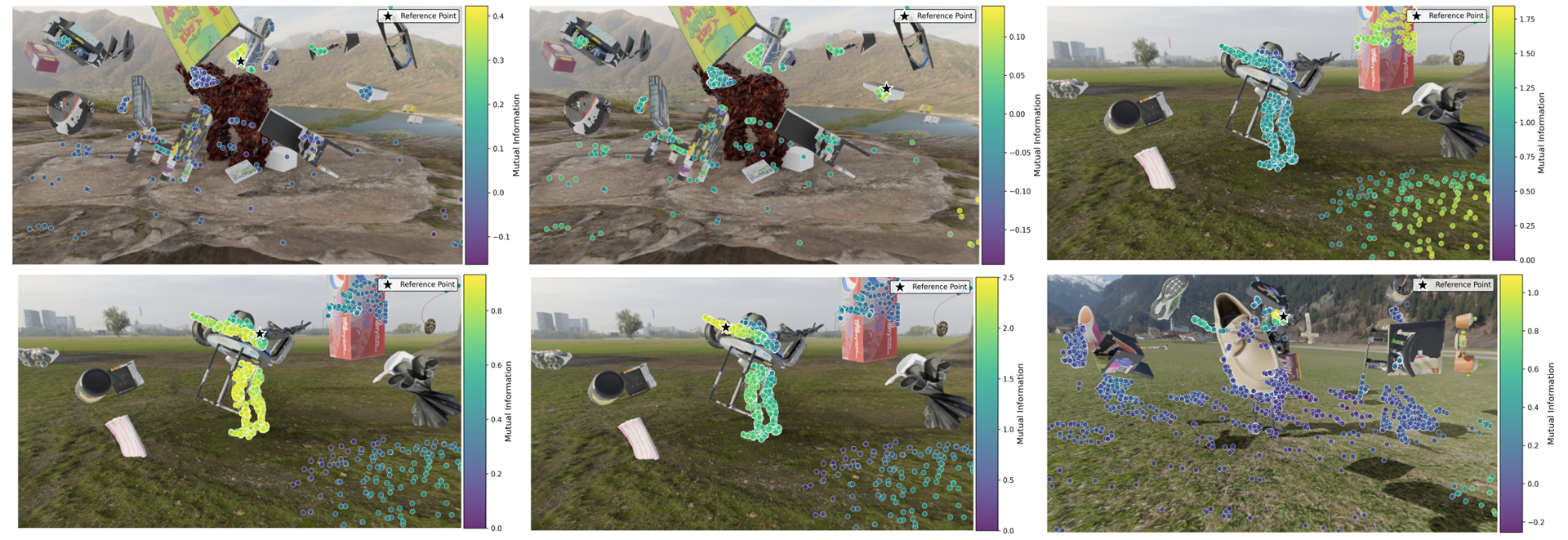}
     \caption{Full visualization results of \method estimated motion data.}
     \label{fig:3d_track_full2}
     \vspace{-8pt}
\end{figure}

\subsection{Details of model training and architecture} \label{Appendix:pre-train}
We provide the details of model architecture and training protocol as below.

\paragraph{Neural architecture details of \method}
For the attention module in \method, we configure the dimensionality of the key and value to 1536. The Weight-Decoding MLP comprises seven layers, each of which has 8196 hidden units. Both the joint-path and marginal-path in \method adopt a Perceiver IO--style architecture~\citep{jaegle2021perceiver}. Specifically, a fixed set of learnable latent queries first performs cross-attention over the input sequence of $n$ samples, and the subsequent self-attention layers are applied only within this latent space. This design avoids quadratic self-attention over the raw input sequence and ensures that both memory usage and computational cost scale linearly with the number of input samples $n$.

\paragraph{Optimizer setup}
We pre-train \method using the Adam optimizer with its default settings for 800,000 iterations, which takes approximately two weeks.

\paragraph{Batch size and distributional diversity}
Each training batch contains 256 independently sampled distributions with 5,000 samples per distribution (so the sequence length fed to the attention module is 5,000), providing the hypernetwork with sufficient statistical evidence to accurately estimate the optimal critic parameters for each distribution type.

\textbf{Neural network training protocol}. To ensure a fair comparison, all neural estimators MINE, InfoNCE, MINDE are trained
for a maximum of 2,000 epochs with a learning rate of $1\times 10^{-4}$,
employing early stopping if no improvement is observed within 100 epochs. KNIFE is trained with 200 epochs.

\textbf{Computational resource}. We pre-train \method on a server with 16 NVIDIA H800 GPUs, while all downstream evaluations are conducted using a single H800 GPU and 8-core Intel(R) 8480C CPU.

\newpage

\subsection{Failure Cases of Sliced Mutual Information}
\label{app:smi-failures}

\newcommand{\vecx}{\mathbf{x}} 
\newcommand{\vecy}{\mathbf{y}}

In this section, we discuss how sliced mutual information (SMI) may fail to fully characterize certain statistical dependencies. Before all, we would like to emphasize that \emph{slicing itself does not nullify dependence}: under mild conditions, statistical dependence in the original space remains detectable from one-dimensional projections, i.e., $\mi(\vecx;\vecy) \neq 0 
    \Leftrightarrow  
    \smi(\vecx;\vecy) \neq 0$
This property is one of the key reasons why SMI can serve as a useful and scalable measure for high-dimensional dependence assessment. However, SMI could still miss crucial information about dependencies, as demonstrated by the following two failure modes.

\paragraph{Indistinguishability of different dependence structures.}
The first failure mode is that two distinct dependence structures (with different underlying MIs) can induce the same SMI, making their structural differences indistinguishable. Consider two-dimensional jointly Gaussian random variables
\[
    \vecx=(X_1,X_2), 
    \qquad 
    \vecy=(Y_1,Y_2),
\]
where both $\vecx$ and $\vecy$ have identity covariance matrices, and their cross-covariance matrix is diagonal. We compare the following two models:

\begin{enumerate}[leftmargin=*]
    \item \emph{Sparse dependence}. The dependence is concentrated on a single coordinate:
    \[
        \operatorname{corr}(X_1,Y_1)=\rho,
        \qquad
        \operatorname{corr}(X_2,Y_2)=0.
    \]

    \item \emph{Dense dependence}. The dependence is evenly distributed across the two coordinates:
    \[
        \operatorname{corr}(X_1,Y_1)
        =
        \operatorname{corr}(X_2,Y_2)
        =
        \rho'.
    \]
\end{enumerate}

Although these two models have different dependence structures, they can induce exactly the same SMI. In particular, by setting
\[
    \rho'
    =
    \sqrt{
    1-
    \left(
    2e^{-\smi_{\mathrm{sparse}}(\rho)}-1
    \right)^2
    },
\]
we obtain
\[
    \smi_{\mathrm{dense}}(\rho')
    =
    \smi_{\mathrm{sparse}}(\rho).
\]
However, their full mutual information values are generally different:
\[
   \mi_{\mathrm{dense}}(\rho')
   \neq
   \mi_{\mathrm{sparse}}(\rho).
\]
For instance, when $\rho=1$, we can derive that for $\rho' \approx 0.7963$, the two models have the same SMI. However, in such case, we have
\[
    \mi_{\mathrm{sparse}}=\infty,
    \qquad
    \mi_{\mathrm{dense}}<\infty.
\]
This shows that SMI may preserve whether dependence exists while still losing information about how dependence is organized in the original high-dimensional space.

\paragraph{Dilution of nonlinear dependence under projection.}
The second failure mode arises when the main shared information is nonlinear and intrinsically high-dimensional, but becomes ambiguous or difficult to detect after projection. A simple example is the polar construction
\[
    \vecx = R(\cos\Theta,\sin\Theta), 
    \qquad 
    \vecy = R(\cos\Phi,\sin\Phi),
\]
where $R \sim p(R)$ is a shared latent radius, while $\Theta \sim \mathcal{U}[0,2\pi]$ and $\Phi \sim \mathcal{U}[0,2\pi]$ are independent random angles. In the original two-dimensional space, $\vecx$ and $\vecy$ share information through their common norm:
\[
    \|\vecx\| = \|\vecy\| = R.
\]
Thus, knowing the norm of $\vecx$ immediately determines the norm of $\vecy$, making the shared information explicit and easy to extract.

Now consider the one-dimensional projections of $\vecx$ and $\vecy$ along directions $u$ and $v$. It is straightforward to verify that the projected variables take the form
\[
    u^\top \vecx = R\cos(\Theta'),
    \qquad
    v^\top \vecy = R\cos(\Phi'),
\]
where $\Theta' \sim \mathcal{U}[0,2\pi]$ and $\Phi' \sim \mathcal{U}[0,2\pi]$ are again independent random angles. 

Here, in the projected space, the shared variable $R$ is multiplied by independent angular noise terms. Consequently, recovering the common radius from a single projected scalar becomes difficult: different values of $R$ can induce highly overlapping projected distributions due to the random cosine factors, making the shared information in $R$ much less identifiable. For instance, if $R \in \{1,1.01\}$, then $R$ is exactly recoverable from either $\|\vecx\|$ or $\|\vecy\|$ in the original space, but becomes hard to distinguish from the projected samples because the small difference in radius is easily masked by the multiplicative noises. In this sense, the nonlinear dependence encoded by the equality of norms is ``diluted'' by slicing.

This example shows that SMI can underestimate or obscure dependencies that are clear in the original high-dimensional geometry, but become entangled with nuisance variation after  projections.

\subsection{Sensitivity to slicing numbers $S$ and projection dimensions $k$}
\label{app:sensitivity}

We sweep $k$ and $S$ on CLIP embeddings ($d=1024$); see Fig.~\ref{fig:clip_kdim_sweep}. Two trends emerge. For any fixed $S$, the estimate grows with $k$ as larger projections retain more joint structure, with scalar slicing ($k=1$) remaining saturated as in \eqref{eq:k-smi}. For any fixed $k$, increasing $S$ reduces Monte-Carlo variance but cannot remove the bias of projecting into a $k$-dimensional subspace (Corollary~\ref{cor:k-sliced}). This supports our default $k=5$, $S=25$, and since \method amortizes per-slice estimation, the trade-off can be revisited at inference without retraining.

\begin{figure}[!t]
    \centering
    \includegraphics[width=0.96\textwidth]{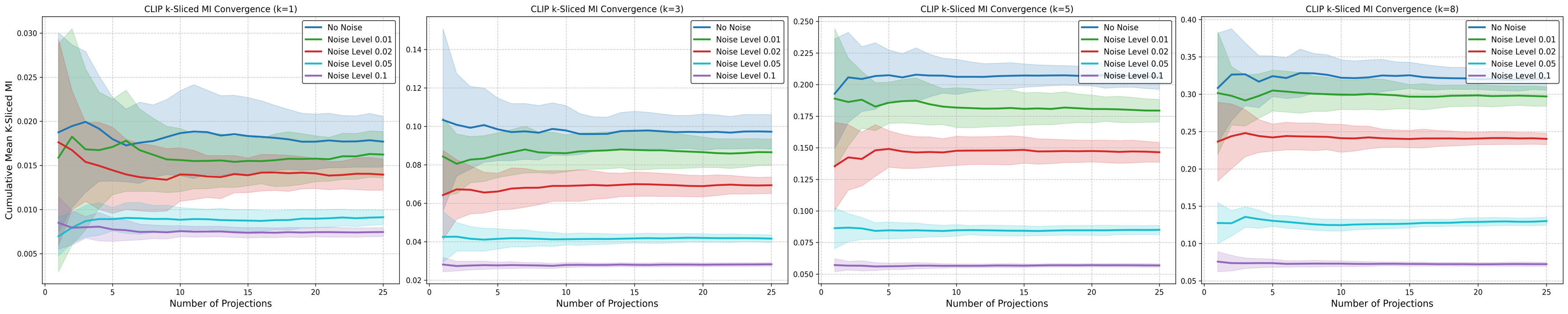}
    \caption{Comparison of slice dimension $k$ and slice number $S$ on CLIP-generated data (original dimension $1024$). Increasing $k$ recovers more ambient dependence per slice, while increasing $S$ reduces Monte-Carlo variance but does not remove the bias introduced by low-dimensional projection.}
    \label{fig:clip_kdim_sweep}
\end{figure}

\section{Limitations}
One limitation of \method is that while \method can effectively handle multivariate data with moderate dimensionalities (e.g. data up to 20 dimensions), it currently requires slicing techniques~\cite{goldfeld2021sliced, goldfeld2022k} to scale to higher dimensions, where extensive pre-training becomes challenging. Despite  relying on slicing in high-dimensional setups, \method is still the first neural method that can directly output MI without iterative optimization for data up to 20D, and it reliably quantifies statistical dependence for data up to 1024D in seconds, as demonstrated by our multiple real-world experiments. Note that for many applications, the exact value of MI is often not the interest; quantifying the orders of statistical dependence is already highly informative.

Another limitation of \method is that it may fall short in cases with small sample cases (e.g. $n < 400$), as seen in the independent testing experiments. In such case, the transformer fail to extract informative signals from a small population. That said, our method remains reliable for typical sample sizes encountered in reality (e.g. $n\geq500$), where our approach consistently matches MINE’s accuracy.


\end{document}